\documentclass{article}
\usepackage[utf8]{inputenc}
\usepackage[margin = 1in]{geometry}

\usepackage{graphicx}
\usepackage{amssymb,amsmath,mathtools,amscd}
\usepackage{amsthm}
\usepackage{epstopdf}
\usepackage{indentfirst}
\usepackage[hidelinks]{hyperref}

\usepackage{amsmath}
\usepackage{mathdots}
\usepackage{yhmath}
\usepackage{cancel}

\usepackage{siunitx}
\usepackage{array}
\usepackage{multirow}
\usepackage{amssymb}
\usepackage{tabularx}
\usepackage{extarrows}
\usepackage{booktabs}

\newcommand{\RR}{\mathbb{R}}
\newcommand{\CC}{\mathbb{C}}
\newcommand{\FF}{\mathbb{F}}

\newcommand{\KK}{\mathbb{K}}

\newtheorem{prop}{Proposition}[section]

\newtheorem{RQ}[prop]{Question}
\newtheorem{theorem}[prop]{Theorem}
\newtheorem{lemma}[prop]{Lemma}
\newtheorem{corollary}[prop]{Corollary}

\theoremstyle{definition}
\newtheorem{rem}[prop]{Remark}

\newcommand{\mfS}{\mathfrak{S}}

\usepackage{subcaption}
\usepackage{booktabs}
\usepackage{multirow}
\usepackage{comment}
\usepackage{subcaption}

\usepackage{authblk}
\title{  A Galois theorem for machine learning: Functions on symmetric matrices and point clouds via lightweight invariant features   }

\author[1]{Ben Blum-Smith}
\author[4]{Ningyuan (Teresa) Huang}
\author[2]{Marco Cuturi}
\author[1,3,4]{Soledad Villar}
\affil[1]{Department of Applied Mathematics and Statistics, Johns Hopkins University}
\affil[2]{Apple}
\affil[3]{Mathematical Institute for Data Science, Johns Hopkins University}
\affil[4]{Flatiron Institute, a division of the Simons Foundation}
\date{}                     
\begin{document}

\maketitle

\maketitle

\begin{abstract}
 In this work, we present a mathematical formulation for machine learning of (1) functions on symmetric matrices that are invariant with respect to the action of permutations by conjugation, and (2) functions on point clouds that are invariant with respect to rotations, reflections, and permutations of the points. To achieve this,   we provide a general construction of generically separating invariant features using ideas inspired by Galois theory.    We construct $O(n^2)$ invariant features derived from generators for the field of rational functions on $n\times n$ symmetric matrices that are invariant under joint permutations of rows and columns.  We show that these invariant features can separate all distinct orbits of symmetric matrices except for a measure zero set; such features can be used to universally approximate invariant functions on almost all weighted graphs. For point clouds in a fixed dimension, we prove that the number of invariant features can be  reduced, generically without losing expressivity, to $O(n)$, where $n$ is the number of points.  We combine these invariant features with DeepSets to learn functions on symmetric matrices and point clouds with varying sizes. We empirically demonstrate the feasibility of our approach on molecule property regression and point cloud distance prediction.

\end{abstract}

\section{Introduction} \label{sec:intro}
Many machine learning (ML) applications come equipped with intrinsic symmetry, and incorporating that knowledge into the design of models has often boosted end performance and efficiency. 
For example, by design, graph neural networks can only learn invariant or equivariant functions with respect to node relabeling; convolutional neural networks are translation invariant; while transformers are permutation invariant. 
As a result, the success of symmetry-informed models has spurred interest in group invariant and equivariant classes of functions. 
These developments have been particularly relevant for ML applied to scientific domains, such as molecular chemistry and physics, where known symmetries are imposed by physical law \cite{zhang2023artificial}. 

In invariant and equivariant ML, the hypothesis class of functions is parametrized in such a way that for every choice of parameters, the corresponding function satisfies the prescribed symmetries. 

The earliest models used tools immediately at hand, such as convolution (in convolutional neural networks) and message-passing, aka graph convolution (in graph neural networks) to achieve the desired symmetry. As invariant and equivariant ML has developed, however, there has been an increase in explicit engagement with the mathematical disciplines that have historically housed the study of group actions and their compatible maps---representation and invariant theory. 

In particular, invariant theory may be broadly defined as the study of function classes that obey prescribed symmetry. Consequently, researchers in equivariant and invariant ML have begun to plumb its depths for insights relevant to their applications. At the same time, ML brings a different set of priorities than traditional invariant theory---for example, invariant theory historically restricts its attention nearly exclusively to polynomial functions, whereas in ML, this is an unnatural and undesirable restriction. So there has been a dynamic interchange of ideas.

The present work fits into this program. We consider the problem of learning a permutation-invariant function on (node and edge) weighted graphs, or learning a function on point clouds that is invariant with respect both to permutations and orthogonal transformations. These two closely related learning problems have a wide variety of applications (see Section~\ref{sec:related-work} below). However, the underlying invariant theory is hard; indeed, a complete set of generating invariants would allow one to solve the graph isomorphism problem, so it is at least as hard as graph isomorphism (solvable in quasi-polynomial time \cite{babai2016graph}; unknown if it is solvable in polynomial time).

Recent exciting work of Hordan et. al.  and Rose et. al. \cite{hordan2023complete, rose2023three,  hordan2024weisfeiler} has proposed solutions to these learning problems that achieve universal approximation guarantees. For low-dimensional point cloud data (i.e., points in $\mathbb{R}^d$ for $d$ fixed), the computational complexity of these approaches is polynomial in the size $n$ of the point cloud; more precisely, $O(n^{d+1})$ in \cite{hordan2023complete, rose2023three} and $O(n^d)$ in \cite{hordan2024weisfeiler}. Still, in large-scale data applications, e.g., to astronomy, $n$ can be on the order of $10^{10}$, while $d=3$ \cite{ivezic2019lsst}. So the methods of \cite{hordan2023complete, rose2023three, hordan2024weisfeiler} are prohibitively computationally expensive in such cases.

Motivated by such considerations of computational scalability, we propose a relaxation of universal approximation. At the cost of throwing out a closed, measure zero set of ``bad" point clouds (respectively, weighted graphs), we provide an approach to these learning problems that has a universal approximation guarantee on the remaining ``good" \emph{generic} point clouds (respectively, weighted graphs), and has  computational cost bounded by a constant multiple of the size of the input data.

Our method is to look for computationally tractable invariants that contain the same information as a set of generators for the field of invariant rational functions on the data. It is known that such field generators separate orbits \emph{generically}, that is, away from a Zariski-closed (and thus Euclidean-closed and measure zero) ``bad" subset; thus the proposed invariants do as well. 
  The mathematical techniques we employ are general. If $\Gamma$ is a group acting on a topological measure space $X$, and $G$ is a finite index subgroup of $\Gamma$, we provide a recipe to extend a set of invariants that are generically separating for $\Gamma$ to a set of generically separating invariants for $G$. The construction is inspired by the fundamental theorem of Galois theory, but it is not restricted to polynomials. In particular, it applies to classes of functions generally used in machine learning.
  
A standard argument based on the Stone-Weierstrass theorem then implies that the proposed invariants are universally approximating away from the ``bad" subset. 

The number of invariants is $O(n^2)$ for an $n$-node weighted graph. Since such a graph has $O(n^2)$ weights, this is tight information-theoretically. However, for point clouds in a fixed $\mathbb{R}^d$, some further work is needed to get the number of proposed invariants from $O(n^2)$ down to $O(n)$. To achieve this, we use a low-rank matrix completion result \cite{hamm2020perspectives}, and recent work on computable low-dimensional embeddings of quotient spaces \cite{dym2024low, tabaghi2023universal}; this involves throwing out another ``bad" subset, but it remains measure-zero and closed.

We formally introduce the group actions of interest in this paper. Suppose $X = (X_{ij})$ is a real symmetric $n\times n$ matrix, and $\mathcal{S}(n)$ is the vector space of all such matrices. The symmetric group $\mfS_n$ on $n$ points acts on $\mathcal{S}(n)$ by
\begin{equation}
\pi X := P_\pi X P_\pi^\top, \label{eq.symmetry}
\end{equation}
where $P_\pi$ is the permutation matrix corresponding to the permutation $\pi$. (Throughout this work, when we refer to an action of $\mfS_n$ on $\mathcal{S}(n)$, it is this action unless otherwise noted.) This is the action corresponding to node reordering in a graph. Specifically, an undirected weighted graph $G$ on $n$ nodes can be represented by a symmetric $n \times n$ matrix $X$, known as the adjacency matrix, where $X_{ij}$ describes the edge weight between nodes $i$ and $j$. (The diagonal entries $X_{ii}$ can be viewed as node weights.) When $X$ is binary and the diagonal is zero, we recover the class of simple undirected graphs. 

The order of the nodes is not an intrinsic property of the graph, but rather a choice; therefore, the representation of the graph as a matrix is not unique. The orbit of $X$ under the action of $\mfS_n$ consists of all adjacency matrices corresponding to the graph $G$. In other words, we can identify the space of weighted graphs with $\mathcal{S}(n)/\mfS_n$, where $\mathcal{S}(n)/\mfS_n$ is the quotient of $\mathcal{S}(n)$ by the group action.

The action \eqref{eq.symmetry} is an important group action in the study of point clouds as well. If $V \in \mathbb R^{d \times n}$ is a point cloud of $n$ points in $\mathbb R^d$, then the first fundamental theorem of invariant functions for the orthogonal group states that a function is invariant with respect to orthogonal transformations (i.e. invariant with respect to rotations and reflections) if and only if it is a function of the Gram matrix $X = V^\top V$ \cite[Theorem~2.9A]{weyl1946classical}. A function is invariant with respect to both orthogonal transformations and permutations of the point cloud $V$ if and only if it is a function of $X= V^{\top}V$ that is invariant with respect to the action \eqref{eq.symmetry}. 

The main contributions of this work are as follows:
\begin{itemize}
\item
 
We provide a ``Galois'' theorem for machine learning. If $\Gamma$ is an arbitrary group acting on a topological measure space $X$ and $G$ is a finite index subgroup of $\Gamma$, the theorem shows how to extend a set of invariants that are generically separating for $\Gamma$ to a set of generically separating invariants for $G$. The construction is similar to the fundamental theorem of Galois theory, but it is not restricted to polynomials. In particular, it applies to various classes of functions used in machine learning.
  
    \item   Using our Galois theorem,    we provide $O(n^2)$ invariant features that generically separate real symmetric matrices with respect to the action \eqref{eq.symmetry} (Theorem~\ref{thm:generic-separators}), so they can be used to universally approximate invariant functions on almost all such matrices (Proposition~\ref{prop:universal-approx-for-matrices}).   The construction hinges on an idea like one that appears in \cite{thiery2000algebraic} for a slightly different setting. 
      
    \item We show that the same features, applied to the Gram matrix of a point cloud, also generically separate (and therefore generically universally approximate invariant functions on) point clouds $V\in \RR^{d\times n}$ with respect to the natural $\mathrm{O}(d)\times \mfS_n$ action (Theorem~\ref{thm:Ond-features} and Proposition~\ref{prop:low-rank-universal}).   The result is not surprising, but nevertheless the proof is not trivial and is new to the best of our knowledge.
      
    \item To improve efficiency, we prove that in the point cloud setting we can express an equivalent class of functions using $O(n)$ invariant features for fixed $d$ (Proposition~\ref{prop:Ond-univ-approx}).   Using ideas from low-rank matrix completion, we show that one can use $O(n)$ features to reconstruct the $O(n^2)$ generically separating invariants from above (Theorem \ref{thm:Ond-features}).   
    \item We use the theoretical results described above, in combination with  DeepSets \cite{zaheer2017deep}, to define invariant ML models on symmetric matrices and point clouds that can be applied to varying-size inputs. To achieve universality, the size of the model depends on an upper bound for the size of the input (Section \ref{sec:DeepSets}).

    \item We illustrate the feasibility of this method in two examples (Section~\ref{sec:experiments}). We show that the invariant features can be used to (1) learn molecular properties, and (2) predict the Gromov-Wasserstein distance between point clouds. All the data and code are available in this repository:
\begin{quote}\texttt{https://github.com/nhuang37/InvariantFeatures}.
    \end{quote}
    
\end{itemize}

Above, we have distinguished typographically between big-$O$ notation and the orthogonal group $\mathrm O(d)$. We preserve this typographical distinction throughout.

\section{Related work}\label{sec:related-work}

\textbf{Invariant theory.} Describing the class of functions invariant under a group action is historically the project of invariant theory. Other works to consider, from the point of view of invariant theory, the same or closely related group actions to the ones discussed above  (i.e., on weighted graphs or point clouds), include \cite{aslaksen1996invariants, thiery2000algebraic, kemper2022separating, gripaios2021lorentz, haddadin2021invariant}; the last two of these have applications to ML in mind. These papers compute complete sets of invariants (i.e., generators for the algebra of invariant polynomials), or separating sets (i.e., invariant polynomials whose evaluation can classify orbits to the extent possible) in special cases. However, a complete or even a separating set of invariants for general $n,d$ is expected to be computationally intractable, because such a set would solve the graph isomorphism problem, for which the fastest known algorithm has quasi-polynomial time complexity. The starting point for the present investigation was the recognition that the method of \cite[Theorem~11.2]{thiery2000algebraic}---which gives generators for the field of rational invariants in a closely related situation where the algebra generators remain out of reach---could be used to produce a set of computationally tractable invariants that uniquely identifies a generic (though not a worst-case) symmetric matrix, up to the $\mfS_n$ conjugation action, and is not meaningfully bigger than the size of the input data. 
Other works that use field generation as a relaxation of algebra generation to lower a computational hurdle include \cite{bandeira2017estimation, blum2023degree}. More general work on computation of fields of rational invariants includes \cite{muller1999calculating, kemper2007computation, hubert2007rational, hubert2012rational, kamke2012algorithmic, hubert2016computation, kemper2016using}.

\textbf{Signal processing with symmetry.} There is a line of research in statistical signal processing that considers a signal that has been corrupted both by noise and by transformations randomly selected from a compact group. Examples include {\em multi-reference alignment} \cite{mra-sdp, AbbPerSin17, PerWeBan17, bandeira2020optimal, bendory2022sparse, abas2022generalized}, {\em cryo-electron microscopy} \cite{ss-cryo, sigworth2016principles, singer2018mathematics, bendory2020single, fan2021maximum}, and their generalizations \cite{bendory2022dihedral, bandeira2017estimation, edidin2023orbit, bendory2024sample}. In these works, recovering a generic signal has emerged as an important relaxation of recovering a worst-case signal because it often allows for dramatic savings in computational cost. This inspired  the present approach.

\textbf{Non-polynomial separating invariants for invariant ML.} While invariant theory traditionally trades exclusively in polynomial invariants, applications to machine learning  have motivated the study of non-polynomial invariants with improved numerical stability properties. Sets of separating invariants derived from polynomial separating sets but renormalized to improve numerical stability were proposed in \cite{cahill2020complete, cahill2024stable}. Invariants based on mapping to a canonical orbit representative were proposed in \cite{olver2023invariants}. For the canonical action of the symmetric group by permutations of the coordinates of a real vector space, such a canonical representative is found by sorting; an elaboration on this idea for the action of the same group on several copies of the canonical representation is proposed in \cite{balan2022permutation}. A uniform construction of separating invariants which are convex in general, piecewise-linear for finite groups, and can often be made bi-Lipschitz, was proposed in \cite{cahill2022group},  studied further in \cite{mixon2022injectivity, mixon2023max}, and generalized in \cite{balan2023i, balan2023g}. General questions concerning the existence and optimum distortion of bi-Lipschitz separating invariants were considered in \cite{cahill2023bilipschitz}. Although the present work is not directly concerned with the Lipschitz property, considerations of numerical stability do motivate our use of non-polynomial invariants.

\textbf{Machine learning on weighted graphs.} 
Graph neural networks (GNNs) have emerged as a popular ML tool for learning functions on graphs \cite{hamilton2020graph, bronstein2021geometric}. GNNs are designed to satisfy permutation invariance (for graph-level output), where the choice of enforcing invariance leads to a tradeoff between efficiency and expressivity. For example, one of the most efficient GNN architectures is the class of message-passing GNNs \cite{kipf2016semi, velickovic2017graph, gilmer2017neural}, which uses the same function to process each node neighborhood. However, message-passing GNNs can suffer from expressivity issues when learning global properties of graphs. The expressivity of GNNs has been extensively studied using the graph isomorphism test and comparison with the Weisfeiler-Lehman (WL) algorithms \cite{xu2018powerful, morris2019weisfeiler}---a hierarchy of combinatorial graph invariants (See \cite{huang2021short} for an introduction and \cite{JMLR:v24:22-0240} for a comprehensive survey). To mitigate the expressivity issues, more powerful GNNs have been proposed, including higher-order GNNs \cite{morris2019weisfeiler, maron2018invariant, puny2023equivariant}, subgraph-based GNNs \cite{chen2020can, Thiede2021autobann, bevilacqua2021equivariant, bouritsas2022improving, frasca2022understanding}, and spectral GNNs \cite{lim2022suvrit, huang2022local}, yet they typically incur higher computation costs due to higher-order tensor operations or additional preprocessing. By using intractable high-order tensors, universality results are established in \cite{pmlr-v97-maron19a} for invariant functions and equivariant functions in \cite{keriven2019universal}.\footnote{Intractability is expected here, as it is unknown whether the graph isomorphism problem can be solved in polynomial time.} In contrast, we relax universality to almost-everywhere universality in order to get the number of features to the same order as the size of the data.

\textbf{Machine learning on point clouds. } Machine learning on point clouds has recently attracted increasing attention due to its wide applications in many areas, such as computer vision, material sciences, cosmology, and autonomous driving. Earlier approaches for classification or segmentation tasks mainly focus on modeling permutation invariance among the points \cite{qi2017pointnet, qi2017pointnet++}, whereas recent works \cite{poulenard2019effective, pozdnyakov2024smooth} further enforce invariance from the orthogonal group. Enforcing joint permutation and Euclidean symmetry has been shown to achieve competitive performance for point cloud classification and segmentation \cite{deng2021vector, melnyk2022tetrasphere}, shape alignment \cite{chen2021equivariant}, shape reconstruction \cite{chatzipantazis2023mathrmseequivariant}, and dynamical system modeling \cite{satorras2021n}. While these works demonstrate promising results empirically, none of them provides universal approximation guarantees.

Universality results on approximating invariant or equivariant functions on point clouds (with respect to both permutation and Euclidean symmetries) have been established using polynomials \cite{dym2020universality}, complete distances \cite{widdowson2023recognizing}, or a geometric variant of the Weisfeiler-Lehman test \cite{pozdnyakov2022incompleteness, joshi2023expressive, hordan2023complete, rose2023three, hordan2024weisfeiler}. Moreover, polynomial-time algorithms to achieve universality on point clouds are discussed in these works (as opposed to intractable algorithms for the case of weighted graphs above; this results from the fact that the data is confined to $\mathbb{R}^d$ for fixed $d$). More concretely,  \cite{hordan2023complete} proposed an invariant network inspired by the $2$-WL test that separates all point clouds in $\RR^3$ with $O(n^4)$ computational complexity, which is further extended to $\RR^d$ for any $d > 1$ with complexity $O(n^{d+1})$ in \cite{hordan2023complete, rose2023three}, and improved in \cite{hordan2024weisfeiler} to $O(n^d)$ complexity (using a different network architecture motivated by \cite{maron2019provably}). Most related to our approach, \cite{hordan2023complete} also showed that a more efficient architecture inspired by the $1$-WL test, with $O(n^2)$ computational complexity (for $d=3$), can achieve universality except for an explicitly-described measure-zero set. We reduce the computational complexity to $O(n)$ by applying ideas from low-rank matrix completion.

\textbf{Enforcing symmetries in machine learning models.} More generally, invariant and equivariant machine learning is a very active research area, where researchers incorporate symmetries into the design of machine learning models \cite{bronstein2021geometric, weiler2021coordinate, villar2023towards}. There are many different approaches, including orbit averaging \cite{elesedy2021provably}, frame averaging \cite{puny2021frame}, representation theory \cite{kondor2018clebsch, geiger2022e3nn, fuchs2020se}, group convolutions \cite{cohen2016group, kondor2018generalization, cohen2019general}, weight sharing \cite{ravanbakhsh2017equivariance}, canonicalization \cite{bloem2020probabilistic,  kaba2023equivariance, aslan2023group,olver2023invariants}, and invariant theory \cite{blum2022machine, villar2021scalars, gripaios2021lorentz, haddadin2021invariant, puny2023equivariant, villar2023dimensionless}, among others. The present work considers the invariant theory approach. Specifically, we leverage the invariants of a compact Lie group to obtain invariants of a (finite index) subgroup. Our construction is somewhat related to recent works on \emph{symmetry breaking} \cite{smidt2021finding, lawrence2024improving, xie2024equivariantsymmetrybreakingsets}, which reduces the symmetries of an equivariant function by augmenting the input to the function with additional symmetry-breaking parameters. Such symmetry-breaking parameters can be learned \cite{smidt2021finding}, sampled from a set based on input and output data symmetries \cite{xie2024equivariantsymmetrybreakingsets}, or random noise (for full symmetry breaking) \cite{lawrence2024improving}.

\section{A ``Galois" theorem for machine learning applications}

Let $X$ be a topological measure space with an action by a group $\Gamma$. In our applications, $X$ is a real vector space and $\Gamma$ is a compact Lie group acting orthogonally, although for the purposes of this section, we do not need the action to respect the measure, or even the topological structure; the only purpose these structures serve is to give subsets of $X$ a notion of ``closed and measure zero". Let $f_1,\dots,f_r$ be a set of $\Gamma$-invariant functions from $X$ to some abelian group $\FF$ (which in machine learning applications will generally be $\RR$ or perhaps $\CC$). We say that $f_1,\dots,f_r$ are {\em generically separating for $\Gamma$} if there exists a closed, measure zero, $\Gamma$-invariant subset $B\subset X$ (the ``bad set") such that  $x_1,x_2\in X\setminus B$ must lie in the same orbit of $\Gamma$ if $f_j(x_1)=f_j(x_2)$ for all $j$. 

Now let $G\subset \Gamma$ be a subgroup of finite index. In this section, we prove an elementary theorem that gives a condition on some $G$-invariant functions $f^\star_1,\dots,f^\star_s:X\rightarrow \FF$ so that if $\Gamma$-invariants $f_1,\dots,f_r$ are  generically separating for $
\Gamma$, then also $f_1,\dots,f_r,f^\star_1,\dots,f^\star_s$ are generically separating for $G$. We use this theorem in Section~\ref{sec:invariants-on-S(n)} to prove that certain small sets of invariants generically separate orbits for the actions of interest in this paper.

The theorem proven here can be seen as an analogue to a certain basic theorem of Galois theory, adapted for machine learning purposes.  Indeed, we use Galois theory to prove a parallel result in Section~\ref{sec:inv-clouds}, and in an earlier draft of the present work, we also used Galois theory for the  Section~\ref{sec:invariants-on-S(n)} result. The latter proof is retained in the Appendix for interest.

Suppose $\mathcal{F}$ is a class of functions $X\rightarrow \FF$ that is an abelian group under pointwise addition, is closed under the natural action of $\Gamma$ by $(\gamma f)(x):=f(\gamma^{-1}x)$ for $x\in X, \gamma\in \Gamma$, and such that the nonzero elements of $\mathcal{F}$ have closed, measure-zero vanishing sets. In our applications, $\mathcal{F}$ is the class of polynomial functions. But, for example, if $X$ is a $\CC$- or $\RR$-vector space and $G$ a linear group, then the class of analytic functions also has these properties, and one could also use other classes.\footnote{ For readers who may have been puzzled at the fact that we do not impose any assumptions on the way $G$'s action on $X$ interacts with the latter's topological or measure structure: it is because of the assumptions on $\mathcal{F}$. If $G$'s action on $X$ is very disruptive of the topological and measure structures, then these assumptions on $\mathcal{F}$ are extremely restrictive.} For a natural number $s$, denote $[s]:=\{1,\dots,s\}$. The promised ``Galois theorem" for machine learning is as follows:

\begin{theorem}\label{thm:galois-for-ML}
    Let $X,\Gamma, G,\mathcal{F}$ be as above.  Suppose $f_1,\dots,f_r: X\rightarrow\FF$ are $\Gamma$-invariant functions that are generically separating for $\Gamma$. If $f^\star_1,\dots,f^\star_s:X\rightarrow\FF$ are $G$-invariant functions belonging to $\mathcal{F}$, such that for $\gamma\in \Gamma$ we have
    \[
    \gamma f^\star_j = f^\star_j\text{ for all }j\in [s]\quad \Rightarrow \quad \gamma\in G,
    \]
    then $f_1,\dots,f_r,f^\star_1,\dots,f^\star_s$ are generically separating for $G$.
\end{theorem}

Note that the hypothesis does {\em not} require $f_1,\dots,f_r$ to belong to $\mathcal{F}$, only $f^\star_1,\dots, f^\star_s$. Below, we will apply this theorem in the special case that $s=1$, so there is a single $f^\star$ that is fixed only by the elements of $G$.

\begin{proof}
    Because $f^\star_1,\dots,f^\star_s$ are $G$-invariant,  the stabilizer $\Gamma_j$ of each $f_j$ in $\Gamma$ contains $G$. The hypothesis on the $f^\star_j$'s can be summarized as
    \[
    \bigcap_j \Gamma_j = G.
    \]
    Because, furthermore, $[\Gamma:G]<\infty$, each $\Gamma_j$ has finite index in $\Gamma$. Thus, there are only finitely many functions of the form $\gamma f^\star_j$, $\gamma\in \Gamma$, namely one for each pair $(j, \gamma \Gamma_j)$ consisting of $j\in [s]$ and a left coset of the stabilizer $\Gamma_j$. They all belong to $\mathcal{F}$, because $\mathcal{F}$ is $\Gamma$-stable. Thus, the finitely many functions
    \[
    \gamma f^\star_j - f^\star_j, \; j\in[s], \;\gamma\notin \Gamma_j
    \]
    all belong to $\mathcal{F}$ as well (because it is an abelian group). Furthermore, they are all nonzero because $\gamma\notin \Gamma_j$ for each one. So by the hypothesis on $\mathcal{F}$, they all have closed, measure zero vanishing sets. Let 
    \begin{align} \label{eq:explicit-bad-set} 
    B=\bigcup_{j\in[s], \gamma\in \Gamma\setminus \Gamma_j}\{x\in X:(\gamma f^\star_j-f^\star_j)(x)=0\}
    \end{align}    
    be the union of these; by the above, this is a union of only finitely many distinct closed, measure zero sets, thus it is closed and measure zero. 
    
    Meanwhile, because $f_1,\dots,f_r$ are generically separating for $\Gamma$, we know that there exists another closed, measure zero set $B'$ on the complement of which any two distinct $\Gamma$-orbits are separated by some $f_j$.

    Then $B\cup B'$ is closed and measure zero, and we claim that any $x_1,x_2 \in X\setminus(B\cup B')$ lying in distinct $G$-orbits are separated either by some $f_j$ or by some $f^\star_j$. Indeed, if $x_1,x_2$ lie in distinct $\Gamma$-orbits, then they are separated by some $f_j$; while if they lie in the same $\Gamma$-orbit but distinct $G$-orbits, then there exists $\gamma\in \Gamma\setminus G$ with $x_1 = \gamma x_2$. In the latter case, there exists $j\in [s]$ with $\gamma \notin \Gamma_j$ because $G=\bigcap_j\Gamma_j$, and then 
    \begin{align*}
    f^\star_j(x_1) - f^\star_j (x_2) &= f^\star_j(\gamma^{-1}x_2) - f^\star_j(x_2)\\
    &= (\gamma f^\star_j-f^\star_j)(x_2)\\
    &\neq 0,
    \end{align*}
    where the final inequality is because $x_2\notin B$ (and $B$ contains the vanishing set of $\gamma f^\star_j - f^\star_j$ by definition). So $x_1,x_2$ are separated by $f^\star_j$.
\end{proof}

Here is the sense in which this is a Galois theory analogue. Let $\FF$ be a field (which we think of as a ground field), and let $\alpha_1,\dots,\alpha_r$ be elements of some (unspecified) extension of $\FF$. Let $\mathbb{K}$ be a finite Galois extension of $\FF(\alpha_1,\dots,\alpha_r)$. This extension has a Galois group, say $\Gamma$, and $\FF(\alpha_1,\dots,\alpha_r)=\mathbb{K}^\Gamma$. Now let $G\subset \Gamma$ be a subgroup of $\Gamma$. Then $\mathbb{K}^G$ is a finite extension of $\FF(\alpha_1,\dots,\alpha_r)$, and we may be interested in identifying generators for this extension. If we can find $\alpha^\star_1,\dots,\alpha^\star_s\in \mathbb{K}^G$ such that 
\[
\gamma \alpha^\star_j = \alpha^\star_j\text{ for all }j\in [s]\quad \Rightarrow \quad \gamma\in G,
\]
then the Fundamental Theorem of Galois Theory tells us that $\alpha^\star_1,\dots,\alpha^\star_s$ generate $\KK^G$ as an extension of $\FF(\alpha_1,\dots,\alpha_r)$, or in other words, that
\[
\FF(\alpha_1,\dots,\alpha_r,\alpha^\star_1,\dots,\alpha^\star_s) = \mathbb{K}^G.
\]
In this analogy, the hypothesis that $\alpha_1,\dots,\alpha_r$ generate $\mathbb{K}^\Gamma$, and the conclusion that $\alpha_1,\dots,\alpha_r,\alpha^\star_1,\dots,\alpha^\star_s$ generate $\mathbb{K}^G$, play the role of the hypothesis that $f_1,\dots,f_r$ are generically separating for $\Gamma$ and the conclusion that $f_1,\dots,f_r,f^\star_1,\dots,f^\star_s$ are generically separating for $G$ in Theorem~\ref{thm:galois-for-ML} above. 

This connection is somewhat more than an analogy, because if $\KK$ happens to be the function field of an algebraic variety $X$ over an algebraically closed field $\FF$ of characteristic zero, and  $\alpha_1,\dots,\alpha_r,\alpha^\star_1,\dots,\alpha^\star_s$ are rational functions on $X$, then Rosenlicht's Theorem \cite[Theorem~2]{rosenlicht1956some} tells us that they generate the invariant field  $\KK^\Gamma$, respectively $\KK^G$, if and only if they are generically separating for $\Gamma$, respectively $G$. So Theorem~\ref{thm:galois-for-ML} is actually a consequence of the Fundamental Theorem of Galois Theory plus Rosenlicht's theorem in this special case. 

But the proof of Theorem~\ref{thm:galois-for-ML} shows that the conclusion holds for an elementary reason that evades this algebraic machinery, so it can be applied in situations where the functions $f_j$ and $f^\star_j$ are not rational functions, or $X$ is not an algebraic variety; this flexibility is the source of its potential usefulness to machine learning theory. It also gives an explicit description of the ``bad set" $B$ introduced by the $f^\star_j$'s (equation~\eqref{eq:explicit-bad-set}).

\section{Invariant functions on symmetric matrices}\label{sec:invariants-on-S(n)}

In this section, we give a set of $O(n^2)$  (specifically, $\binom{n+1}{2}+1$) functions that are reasonable to compute in practice and that, in a sense to be made precise momentarily, {\em almost} universally approximate functions on $\mathcal{S}(n)$.

For $k=1,\dots,n$, let 
\[
f^d_k(X)= \text{the $k$th largest of the numbers } X_{11},\dots,X_{nn},
\]
in other words, the $k$th element in a sorted list of these numbers. For $\ell=1,\dots,\binom{n}{2}$, let 
\[
f^o_\ell(X)=\text{the $\ell$th largest of the numbers }X_{ij},\; 1\leq i < j \leq n.
\]
(The superscripts $d$ and $o$ stand for ``diagonal" and ``off-diagonal".) Finally, let 
\[
f^\star(X) = \sum_{i\neq j} X_{ii}X_{ij}.
\]
Collectively we refer to these functions as {\em the $f$'s} (or {\em the $f^d$'s} or {\em the $f^o$'s} for the corresponding subcollections). Note that computation of the $f^d$'s can be accomplished with time complexity $O(n\log n)$ and space complexity $O(n)$ or better, while computation of the $f^o$'s can be done with time complexity $O(n^2\log n)$ and space complexity $O(n^2)$ or better. Meanwhile, by summing first over $j$, $f^\star(X)$ can be computed with time complexity $O(n^2)$ and space complexity $O(n)$. Thus, the computational complexity of the proposed $f$'s is bounded above (up to log factors) by the order of growth of the input data, which is $\Theta(n^2)$.

A {\em multi-layer perceptron} (the basic model for machine learning architectures) is a composition of functions between various finite-dimensional real  vector spaces, of the form
\[
\rho \circ L_{\ell} \circ \rho \circ \dots \circ \rho \circ L_1,
\]
where the $L_1,\dots, L_\ell$ are linear maps that are learned from data, and $\rho$ is a fixed, continuous but non-polynomial function $\RR\rightarrow \RR$ applied entrywise to the coordinates of the vectors. 

The purpose of this section is the following:

\begin{prop}\label{prop:universal-approx-for-matrices}
There is a closed, $\mfS_n$-invariant, measure zero set $B\subset \mathcal{S}(n)$ (the ``bad set"), such that any continuous, $\mfS_n$-invariant function $\mathcal{S}(n)\setminus B\rightarrow \RR$ defined on the complement of $B$ can be uniformly approximated on any compact subset by a multi-layer perceptron that takes the $f$'s as input.
\end{prop}

\begin{rem}
    The ``bad set" $B$ is defined by the vanishing of some polynomials; in other words, it is an algebraic subvariety of $\mathcal{S}(n)$.
\end{rem}

\begin{rem}
    The argument follows a standard pattern in machine learning of translating an orbit-separation result (Theorem~\ref{thm:generic-separators} below) into a universal approximation result via the Stone-Weierstrass theorem.
\end{rem}

\begin{proof}[Proof of Proposition~\ref{prop:universal-approx-for-matrices}]
All  of the $f^d$'s and all of the $f^o$'s are continuous because they are piecewise-linear; $f^\star$ is continuous as it is a quadratic polynomial. Meanwhile, all of the $f$'s are constant along orbits of $\mfS_n$. Thus, they induce continuous functions
\[
\mathcal{S}(n)/\mfS_n \rightarrow \RR,
\]
where $\mathcal{S}(n)/\mfS_n$ is the topological quotient of $\mathcal{S}(n)$ by the group action. Since $\mathcal{S}(n)$ is hausdorff and $\mfS_n$ is compact, the quotient is hausdorff \cite[Theorem~I.3.1]{bredon1972introduction}. Therefore, by the Stone-Weierstrass theorem applied to this quotient, the algebra generated by the functions induced on $\mathcal{S}(n)/\mfS_n$ by the $f$'s uniformly approximates any continuous function on any compact set on which these functions separate points in the quotient $\mathcal{S}(n)/\mfS_n$. Pulling back to $\mathcal{S}(n)$, this means that the algebra generated by the $f$'s uniformly approximates any continuous invariant function on any compact set where the $f$'s separate orbits. It is standard \cite[Theorem~1]{cybenko1989approximation} that a multi-layer perceptron trained on the $f$'s can approximate arbitrary continuous functions in the $f$'s, and in particular, any element of the algebra generated by the $f$'s. Thus, the proposition follows from the following theorem, which will complete the proof.
\end{proof}

\begin{theorem}\label{thm:generic-separators}
There is a closed, $\mfS_n$-invariant, measure zero set $B\subset \mathcal{S}(n)$ such that any two distinct orbits of $\mfS_n$ in the complement of $B$ are separated by some $f$.
\end{theorem}

We give a proof using Theorem~\ref{thm:galois-for-ML}. A previous version of this paper gave a Galois theory-based proof that is preserved in the Appendix.

\begin{proof}  
The proof uses the main idea from \cite[Theorem~11.2]{thiery2000algebraic}, and the above Theorem~\ref{thm:galois-for-ML}.    The $f^d$'s and $f^o$'s are invariant with respect to an arbitrary permutation of the $X_{ii}$'s and an arbitrary (unrelated) permutation of the $X_{ij}$'s ($i\neq j$); thus, they are invariants of a bigger group action on $\mathcal{S}(n)$ than the action of $\mfS_n$ described above. Specifically, let $\mfS_n$ act by arbitrary permutations on the $X_{ii}$'s while fixing the $X_{ij}$'s ($i\neq j$), and let $\mfS_{n(n-1)/2}$ act by arbitrary permutations on the $X_{ij}$'s for $i\neq j$ (that respect the relation $X_{ij}=X_{ji}$, i.e., that preserve the symmetry of the matrix $(X_{ij})$) while fixing the $X_{ii}$'s. Since these actions commute with each other (as they are by permutations supported on disjoint sets),  they induce an action by the direct product $\Gamma :=\mfS_n\times \mfS_{n(n-1)/2}$ on $\mathcal{S}(n)$. The $f^d$'s and $f^o$'s are invariant functions for this (bigger) action.

Furthermore, the $f^d$'s uniquely characterize the numbers $X_{11},\dots,X_{nn}$ up to order;  similarly,    the $f^o$'s uniquely characterize the $\binom{n}{2}$ numbers $X_{ij}$, $1\leq i<j\leq n$ up to order. Thus, taken as a whole, the $f^d$'s and $f^o$'s yield a set of orbit-separating invariants for the action by $\Gamma$ on $\mathcal{S}(n)$ described in the last paragraph. 

We would be done (with no need to exclude a ``bad set" $B$) if our interest were in this action by $\Gamma = \mfS_n\times\mfS_{n(n-1)/2}$. However, our actual interest is in the action of $\mfS_n$ given by $\pi X := P_\pi XP_\pi^\top$ per \eqref{eq.symmetry}. In the latter action, a permutation $\pi\in \mfS_n$ acts on both the $X_{ii}$'s and the $X_{ij}$'s at once by permuting the underlying index set over which $i$ and $j$ vary.   Let $G$ be the subgroup    of $\Gamma$   such that    the entries in a pair $(\sigma,\tau)\in \Gamma = \mfS_n\times \mfS_{n(n-1)/2}$ are induced by the same underlying permutation of the index set. (Note that $G\cong \mfS_n$, but we introduce the new notation so as to think of it as a subgroup of $\Gamma$.)

  We now use Theorem~\ref{thm:galois-for-ML}. Let $\mathcal{F}$ be the class of polynomial functions on $\mathcal{S}(n)$. This family forms a vector space that is stable under the action of $\Gamma$, and nonzero polynomials have closed and measure zero vanishing sets. We have $f^\star\in \mathcal{F}$; to verify the hypothesis of Theorem~\ref{thm:galois-for-ML} we need to also check that $\gamma f^\star = f^\star \Rightarrow \gamma\in G$ for all $\gamma\in \Gamma$.   For an element $(\sigma,\tau)$ of $\Gamma$ to leave $f^\star$ invariant, it must send terms of the form $X_{ii}X_{ij}$ to other such terms. So the pair $(\sigma,\tau)$ of permutations must preserve the relation of sharing an index between the $X_{ii}$'s and the $X_{ij}$'s ($i\neq j$). Now $\sigma$ (acting on the $X_{ii}$'s) tells us what the underlying permutation of the indices must be, and the fact that the pair $(\sigma,\tau)$ preserves this incidence relation tells us that $\tau$ must be induced by the same underlying permutation of the indices. (This point is argued   in more detail    in the   original    proof in the Appendix.) Thus the only elements of $\Gamma$ that leave $f^\star$ invariant are precisely those belonging to $G$, and  Theorem~\ref{thm:galois-for-ML} then lets us conclude that the $f^d$'s, $f^o$'s, and $f^\star$ are generically separating for $G$. 
\end{proof}

\section{Invariant functions on point clouds} \label{sec:inv-clouds}

In this section, we show that the invariant features given in the previous section are also able to generically separate point clouds up to orthogonal transformation and relabeling, when viewed as functions on the Gram matrix of the point cloud. As before, it follows that these invariant features are able to approximate arbitrary continuous functions on such point clouds. A little additional work is needed here because the previous section's result holds only generically, but the Gram matrix of a (generic) point cloud in fixed dimension $d$ is not a generic symmetric matrix.

A defect of the result given here is that whereas a symmetric matrix is specified by $O(n^2)$ numbers, a point cloud is specified by $O(n)$ numbers (assuming that the dimension $d$ of the ambient space is fixed). Thus the complexity of computing the proposed invariant features scales with the square of the input data (up to log factors), rather than linearly. We rectify this defect in Section~\ref{sec:reduction}.

We are interested in point cloud learning scenarios in which the ambient space of the point cloud does not carry any  natural orientation (e.g., galaxy point clouds in astronomy). Thus, in this work, a point cloud is a tuple $v_1,\dots, v_n$ of vectors in $\RR^d$, modulo equivalence under the action of the orthogonal group $\mathrm O(d)$ simultaneously on all the vectors, in addition to permutations of the vectors.\footnote{In astronomical applications, the origin is also arbitrary. One can factor out this translation invariance by replacing each $v_i$ with $v_i-\frac{1}{n}\sum_j v_j$.} If we amalgamate the vectors into a $d\times n$ matrix $V:=\begin{pmatrix}v_1&\dots&v_n\end{pmatrix}\in \RR^{d\times n}$, then the equivalence relation is given by
\[
V\sim UVP_\pi^\top,
\]
where $U\in \mathrm O(d)$ is an arbitrary $d\times d$ orthogonal matrix, and $P_\pi$ is an $n\times n$ permutation matrix corresponding to an arbitrary permutation $\pi\in \mfS_n$; thus the space of point clouds is the orbit space
\[
\mathrm O(d) \backslash \RR^{d\times n}/\mfS_n,
\]
with $\mathrm O(d)$ acting on the left and $\mfS_n$ acting on the right as above.

There is a natural map
\[
\gamma:\mathrm O(d)\backslash \RR^{d\times n}/\mfS_n\rightarrow \mathcal{S}(n)/\mfS_n,
\]
where the action of $\mfS_n$ on $\mathcal{S}(n)$ is given by \eqref{eq.symmetry}. This map is induced from the map
\[
\widehat \gamma: \RR^{d\times n}\rightarrow \mathcal{S}(n)
\]
given by
\[
V\mapsto V^\top V
\]
which sends a tuple of vectors to its Gram matrix. Note that $\widehat \gamma:\RR^{d\times n}\rightarrow\mathcal{S}(n)$ is $\mathrm O(d)$-invariant by construction, so it factors through the quotient space $\mathrm O(d) \backslash \RR^{d\times n}$; and it is also $\mfS_n$-equivariant by construction, thus it induces the claimed map $\gamma$ on the quotients by $\mfS_n$.

The image of $\widehat \gamma$ in $\mathcal{S}(n)$ is the set of positive semidefinite matrices of rank $\leq d$. The $f^d$'s, $f^o$'s, and $f^*$ of the previous section become functions on $\RR^{n\times d}$ via composition with $\widehat \gamma$; we indicate these composed maps by adding a tilde to the $f$'s, as in,
\begin{align*}
    \tilde f^*: \RR^{d\times n}&\rightarrow \RR\\
    V&\mapsto f^*(\widehat \gamma(V)) = f^*(V^\top V),
\end{align*}
and similarly for the $\tilde f^d$'s and $\tilde f^o$'s (which we refer to collectively as ``the $\tilde f$'s"). These functions are $\mathrm O(d)\times \mfS_n$-invariant per the discussion above. A priori, because of the inexplicitness of the ``bad set" $B$ of the previous section, there is a danger that the set of rank $\leq d$ positive semidefinite matrices lies within it, so we cannot apply Proposition~\ref{prop:universal-approx-for-matrices} or Theorem~\ref{thm:generic-separators} directly in this setting. However, with some additional care, we can prove the following. 

\begin{prop}\label{prop:low-rank-universal}
    There is a closed, $\mathrm O(d)\times \mfS_n$-invariant, measure zero set $\tilde B\subset \RR^{d\times n}$ (the ``bad set"), such that any continuous, $\mathrm O(d)\times \mfS_n$-invariant function $\RR^{d\times n}\setminus \tilde B\rightarrow \RR$ defined on the complement of $B$ can be uniformly approximated on any compact subset by a multi-layer perceptron that takes the $\tilde f$'s as inputs.
\end{prop}

\begin{proof}
    This follows from the following theorem just as Proposition~\ref{prop:universal-approx-for-matrices} followed from Theorem~\ref{thm:generic-separators}.
\end{proof}

\begin{theorem}\label{thm:low-rank-separating}
    There is a closed, $\mathrm O(d)\times \mfS_n$-invariant, measure zero set $\tilde B\subset\RR^{d\times n}$ such that any two distinct orbits of $\mathrm O(d)\times \mfS_n$ in the complement of $\tilde B$ are separated by some $\tilde f$.
\end{theorem}

The proof is in the Appendix. The argument is similar in spirit and general structure to the proof of Theorem~\ref{thm:generic-separators}, but is more involved. The main additional ideas are as follows:
\begin{enumerate} 
    \item The action by the bigger group $\Gamma$ of Theorem~\ref{thm:generic-separators} is on $\mathcal{S}(n)$, not on $\RR^{d\times n}$. To get around this, we work in an algebraic setting so we can invoke some standard facts about field automorphisms. Thus, in place of Theorem~\ref{thm:galois-for-ML} we use Galois theory to compute generators for the field of rational invariants of the action of $\mathrm{O}(d)\times \mfS_n$ on $\RR^{d\times n}$ described above---see Corollary~\ref{cor:pt-cloud-field-generators}---and we use Rosenlicht's theorem to arrive at generic separation.   
    
    \item The ring of polynomial functions on  point clouds modulo the action of $\mathrm O(d)$ is isomorphic to the quotient of the ring of polynomial functions on symmetric matrices, modulo the  ideal $I_{d+1}$ of $(d+1)\times (d+1)$ minors. The field of rational functions is the fraction field of this.
    \item As with the $f^d$'s and $f^o$'s of Theorem~\ref{thm:generic-separators}, the $\tilde f^d$'s and $\tilde f^o$'s identify the Gram matrix of the point cloud up to arbitrary independent permutations of the diagonal and off-diagonal elements. So the argument again comes down to checking that the only ones of these pairs of permutations that fix $\tilde f^*$ are those that are induced by the same underlying permutation of the original points. But this time, to check this, we need to make sure that all other pairs of permutations not only fail to fix $f^*$, but even fail to fix it modulo $I_{d+1}$. So we need to reason a little about the ideal $I_{d+1}$.
    \item For $d>1$, it is enough to observe that $I_{d+1}$ is generated by elements of degree $>2$. But for $d=1$ (admittedly a minor case anyway), we reason based on the actual form of the generators for $I_{d+1}$. The main idea for this case was suggested to us by Alexandra Pevzner.
\end{enumerate}

\section{  Reducing the number of invariants on point clouds}
\label{sec:reduction}
In this section, we show that for point clouds of $n$ points in dimension $d$ ($d\ll n$), we can reduce the number of invariant features and overall computational complexity from $O(n^2)$ to $O(n)$, by exploiting that the Gram matrix is low rank   (see Theorem \ref{thm:Ond-features}).    
To this end, we make use of two objects: \emph{$d$ equivariant identifiers} and \emph{permutation-orbit separators}, defined as follows.

\begin{description}
    \item[$\mathbf d$ equivariant identifiers:] We consider a function $C: \mathbb R^{d\times n} \to \mathbb R^{d\times d}$, that is defined and continuous everywhere, except possibly in an $\mathrm{O}(d)\times \mfS_n$-stable measure zero closed set. We further ask that for all input $V$ (except possibly in an additional  $\mathrm{O}(d)\times \mfS_n$-stable closed measure zero set), this function outputs a matrix $C_V$ whose columns are an ordered $d$-tuple of linearly independent points in $\mathbb R^n$. Finally, we ask that $C_V$ be an $\mathrm O(d)$-equivariant and $\mfS_n$-invariant function of $V$ where defined, i.e.,
    \[
    C_{UVP_\pi^\top} = UC_V,
    \]
    where $U\in \mathrm O(d)$ and $\pi \in \mfS_n$.
\end{description}

\begin{rem}\label{rem:d-equivariant-identifiers}   The $d$ equivariant identifiers is a generic concept that we use as a computational tool in order to reduce the number of invariant features from $O(n^2)$ to $O(n)$. Standard concepts such as $k$-means centers satisfy this definition even if the data is not clusterable, and can be used as a black-box subroutine. In particular,    Proposition \ref{prop.kmeans} below shows that we can take $C_V$ to be the matrix with columns the $k$-means centroids of the point cloud $V$ with $k=d$, ordered by the 2-norm. Using $k$-means clustering for quantizing point clouds is a standard practice in data science (e.g.\cite{blumberg2020mrec,beugnot2021improving}). In theory, computing the $k$-means centroids is an NP-hard problem \cite{awasthi2015hardness}, but several randomized algorithms (such as kmeans++ \cite{arthur2007k}) are very efficient and seem to perform well in practice. However, centroids obtained via randomized algorithms do not exactly fit the theoretical framework we consider here (since we cannot guarantee that the centroids will be a function of the input, let alone $\mathrm O(d)$-equivariant or $\mfS_n$-invariant). There exist deterministic approximation algorithms for the $k$-means centroids that either are or can be modified to be  $\mathrm O(d)$-equivariant and $\mfS_n$-invariant, for example the method defined in Theorem 14 of \cite{mixon2017clustering}. For simplicity, we prove Proposition \ref{prop.kmeans} for the $k$-means optimal centroids (which we may not be able to compute in general), and, also for simplicity, in the numerical experiments in Section \ref{sec:experiments} we use the standard (randomized) implementation of $k$-means. For the sake of theoretical completeness, we also exhibit an alternative construction for $d$ equivariant identifiers in Proposition~\ref{prop:alternative-C}, which are directly computable with $O(n)$  cost (for fixed $d$).
\end{rem}

\begin{description}

\item[Permutation-orbit separators:] We consider
a continuous function $D:\mathbb R^{d\times n} \to \mathbb R^{s}$, where $s = O(n)$ (for fixed $d$), satisfying: 
 \begin{enumerate}
     \item[(a)] $D$ is $\mfS_n$-invariant for the permutation action on columns, i.e.,
     \[
     D(VP_\pi^\top) = D(V)
     \] 
     for $\pi \in \mfS_n$; and
     \item[(b)] $D$ is  separating: for every $x\in \mathbb R^{d\times n}$,  if $y\in \mathbb R^{d\times n}$ satisfies $D(x)=D(y)$, then $x$ and $y$ are in the same $\mfS_n$-orbit.
 \end{enumerate}
 \end{description}
 \begin{rem}
A method for constructing permutation-orbit separators $D$ as defined above is given in \cite[Proposition~2.1]{dym2024low}, with $s=2nd+1$. In practice, we will implement learning models belonging to a function class known as DeepSets \cite{zaheer2017deep}, widely used in machine learning. We describe DeepSets in the next section, and use these models for the numerical experiments in Section \ref{sec:experiments}. Theorems 5 and 6 in \cite{tabaghi2023universal} imply\footnote{This implication is not drawn out explicitly in \cite{tabaghi2023universal}. Theorem 5 only states that there exist {\em generically} separating DeepSets. However, Theorem 6 cannot hold unless the DeepSets of Theorem 5 are in fact fully separating.} that there exist DeepSets with an encoder $\phi$ (see Section~\ref{sec:DeepSets} for notation) such that the encode-plus-aggregate step $\{x_i\}\mapsto \sum_i\phi(x_i)$ serves as a permutation-orbit separator with $s=2nd$. This result should be interpreted as saying that DeepSets with $O(n)$ latent dimension are expressive enough to serve the function of the permutation-orbit separator. The proof technique requires a full $2nd$ dimensions in the latent space---but as a matter of engineering, for learning a particular target function one may obtain better results using a smaller latent dimension $s$.
\end{rem}

\begin{rem}
    It is well-known in invariant theory, in a more general context, that there always exists a separating set of size approximately twice the dimension of the data space \cite[Proposition~5.1.1]{dufresne-thesis}, \cite[Theorem~6.1]{kemper2024invariant}.  However, the proofs involve computing a potentially much larger separating set and projecting it to low dimension, so these theorems do not provide separators computable with $O(n)$ space or time cost. Similarly, \cite[Theorem~12 and Section 3.2.1]{cahill2022group} provides a methodology for finding twice-dimension-many worst-case separators, but again they are not (at least evidently) computable in $O(n)$ time. The advantage of \cite[Theorems~5 and 6]{tabaghi2023universal} for our purposes is that they show that the DeepSets architecture, which is practical for other reasons---including that it allows to uniformly handle point clouds of different sizes---can serve the function of the permutation-orbit separators, with not worse than the target computational cost.
\end{rem}

Using these ingredients   we provide a construction of $O(n)$ generically separating invariants for $d$-dimensional point clouds with $n$ points.    We can prove the following statement. 

\begin{theorem}\label{thm:Ond-features}
Given $V \in \mathbb R^{d \times n}$, $C:\mathbb R^{d\times n} \to \mathbb R^{d\times d}$ and $D:\mathbb R^{d\times n} \to \mathbb R^s$ as above, then there exists a closed, $\mathrm O(d)\times \mfS_n$-stable, measure zero set $\widehat B\subset\RR^{d\times n}$ such that in the complement of $\widehat B$, the map
\[
h(V):= (D (C_V^\top V), \;  C_V^\top C_V) \in \mathbb R^s \times \mathcal{S}(d)
\]
is defined, continuous, $\mathrm O(d)\times \mfS_n$-invariant, and separates any two distinct orbits of $\mathrm O(d)\times \mfS_n$.
\end{theorem}

\begin{proof}
Given a point cloud $V\in \mathbb R^{d\times n}$, consider $\tilde V = 
\left[\begin{matrix}
V & C_V
\end{matrix}
\right]  \in \mathbb R^{d \times (n+d)}$, and $\tilde X = \tilde V^\top \tilde V \in \mathbb R^{(n+d) \times (n+d)}$.
The matrix $\tilde X$ is a rank-$d$ symmetric matrix.
Theorem A in \cite{hamm2020perspectives} implies that $\tilde X$ can be reconstructed from the last $d$ rows and $d$ columns of $\tilde X$ as long as $C_V$ is full rank. We are interested in constructing a continuous separator $h$ from $C_V$, so we also ask the latter to be a continuous function of $V$. Both these properties of $C_V$ are guaranteed for $V$ lying outside of some $\mathrm{O}(d)\times \mfS_n$-stable, closed, measure zero set $B'$, by the definition of $C$. 

We take $\widehat B$ to be the union of $B'$ and the $\tilde B$ of Theorem~\ref{thm:low-rank-separating}. It is clear that $h$ is defined and continuous on the complement of $\widehat B$ (and in fact even on the complement of $B'$). Similarly, $h$ is $\mathrm O(d)\times \mfS_n$-invariant on $\RR^{d\times n}\setminus \widehat B$, because given $(U,\pi)\in \mathrm O(d)\times \mfS_n$, we have
\begin{align*}
h(UVP_\pi^\top) &= (D(C_V^\top U^\top UVP_\pi^\top), C_V^\top U^\top UC_V) \\
&= (D(C_V^\top V), C_V^\top C_V)\\
&= h(V),
\end{align*}
by the orthogonality of $U$ and the properties of $C$ and $D$. We need to show $h$ is separating on $\RR^{d\times n}\setminus \widehat B$.

Suppose $V_1,V_2\in \RR^{d\times n}\setminus \widehat B$ lie in different orbits of $\mathrm{O}(d)\times \mfS_n$. Then they are distinguished by some $\tilde f$, by Theorem~\ref{thm:low-rank-separating} (where the $\tilde f$'s are as defined in Section~\ref{sec:inv-clouds}). Since the $\tilde f$'s can be computed from the Gram matrices $V_1^\top V_1$ and $V_2^\top V_2$, and are invariant with respect to the $\mfS_n$-conjugation action on these matrices, this implies that $V_1^\top V_1$ and $V_2^\top V_2$ lie in different orbits for the $\mfS_n$-conjugation action. 

In turn, this implies that, with $\tilde V_i = \left[\begin{matrix}
V_i & C_{V_i}
\end{matrix}
\right]$ and $\tilde X_i = \tilde V_i^\top \tilde V_i$ as above, where $i=1,2$, we must have $\tilde X_1$ and $\tilde X_2$ lying in different orbits for the conjugation action by $\mfS_n$ on the first $n$ rows and columns.

By the matrix reconstruction result \cite[Theorem~A]{hamm2020perspectives} quoted above, this then implies that $C_{V_1}^\top\tilde V_1$ and $C_{V_2}^\top\tilde V_2$ are different, since these matrices constitute the last $d$ rows of $\tilde X_1,\tilde X_2$ and so by symmetry, also the last $d$ columns. In fact, since the conjugation action of $\mfS_n$ on the first $n$ rows and columns of each $\tilde X_i$ has the effect of permuting the first $n$ columns of the final $d$ rows (without permuting these rows), it further follows that $C_{V_1}^\top\tilde V_1$ and $C_{V_2}^\top\tilde V_2$ lie in different orbits for the permutation action of $\mfS_n$ on the first $n$ columns. In more explicit terms, either $C_{V_1}^\top\tilde V_1$ and $C_{V_2}^\top\tilde V_2$ differ in their final $d$ columns---which are the columns of $C_{V_1}^\top C_{V_1}$ and $C_{V_2}^\top C_{V_2}$---or they differ even up to column permutation in their first $n$ columns---which are the columns of $C_{V_1}^\top V_1$ and $C_{V_2}^\top V_2$---or both.

Since $D$ is a permutation-orbit separator, if $C_{V_1}^\top V_1$ and $C_{V_2}^\top V_2$ differ even up to column permutation (i.e., lie in different orbits for the $\mfS_n$-action on columns), then this difference is detected by $D$. Thus, either
\[
C_{V_1}^\top C_{V_1}\neq C_{V_2}^\top C_{V_2},
\]
or
\[
D(C_{V_1}^\top V_1) \neq D(C_{V_2}^\top V_2),
\]
or both. In other words,
\[
h(V_1)\neq h(V_2). \qedhere
\]
\end{proof}

As in Sections~\ref{sec:invariants-on-S(n)} and \ref{sec:inv-clouds}, we immediately get a universal approximation guarantee on the complement of $\widehat B$, with the same proof:

\begin{prop}\label{prop:Ond-univ-approx}
    With $h, V$ and $\widehat B$ as in Theorem~\ref{thm:Ond-features}, any $\mathrm O(d)\times \mfS_n$-invariant function $\RR^{d\times n}\setminus \widehat B\rightarrow \RR$ can be uniformly approximated on any compact subset by a multi-layer perceptron that takes $h(V)$ as input. \qed
\end{prop}

\begin{rem}
    In Theorem~\ref{thm:Ond-features}, we formed the matrix $C_V^\top V$, and then processed its columns as a multiset (i.e., in a manner that is invariant to column permutations). This is related to a device used in \cite[Theorem~4]{hordan2023complete}, in which $d-1$ points from the original point cloud $V$ are selected arbitrarily, a $d\times d$ matrix $X$ is formed from them in an $\mathrm O(d)$-equivariant manner, and the columns of $X^\top V'$ are then processed as a multiset, where $V'$ is $V$ with the selected points omitted. To obtain permutation invariance, the method in \cite{hordan2023complete} requires to repeat this process over all $\binom{n}{d-1}$ choices of $d-1$ points; the present approach avoids this computational burden by using a single permutation-invariant $C$.
\end{rem}

We close the section with two propositions promised by Remark~\ref{rem:d-equivariant-identifiers}. In Proposition~\ref{prop.kmeans}, we verify the claims made in that remark about $k$-means clustering as a choice for the $d$ equivariant identifiers. We note that $k$-means clustering is just a standard method that can be used to produce $d$ equivariant identifiers regardless of whether the data is naturally clusterable. Then, in Proposition \ref{prop:alternative-C}, we provide an alternative implementation of the $d$ equivariant identifiers that does not rely on $k$-means clustering.

\begin{prop} \label{prop.kmeans}
    There is a closed, measure zero, $\mathrm O(d) \times \mfS_n$-stable set $B^!$ in $\RR^{d\times n}$ (a ``bad set"), such that we have the following on its complement:
    \begin{enumerate}
        \item For a point cloud $V\in \RR^{d\times n} \setminus B^!$, 
        the $k$-means clustering centroids (taking $k=d$) have distinct norms. Thus, there is a well-defined function
    \[
    C : \RR^{d\times n}\setminus B^! \rightarrow \RR^{d\times d}
    \]
    that, given a point cloud $V$, returns the matrix $C_V$ whose columns are the $k$-means clustering centroids $c_1, \ldots, c_d \in \mathbb R^d$, ordered by descending $2$-norm: $\|c_1\|>\|c_2\|>\ldots >\|c_d\|$.\label{item:def-of-C}
    \item The function $C$ defined in \ref{item:def-of-C} is $\mathrm O(d)$-equivariant, $\mfS_n$-invariant, and continuous (on $\RR^{d\times n}\setminus B^!$).
    \end{enumerate} 
\end{prop}

\begin{proof}
Let $V= \left(\begin{matrix}
| & & | \\
p_1&\cdots &p_n\\
| & & |
\end{matrix}
\right) \in \mathbb R^{d\times n}$. Let $\mathcal P_k(n)$ be the (finite) set of partitions of the $n$ columns of $V$ into $k$ disjoint subsets. 
The $k$-means problem looks for the partition of the points into clusters that minimizes the sum of the squared distances from each point to the centroid of its cluster, 
\begin{align}
    \operatorname{kmeans-val}(V) = \min_{\mathcal A: = \{A_1,\ldots,A_k\} \in \mathcal P_k(n)} \operatorname{kmeans-obj}(\mathcal A, V)\\ 
    \operatorname{kmeans-obj}(\mathcal A, V) = \sum_{s=1}^k \sum_{i\in A_s} \|p_i - c_s(\mathcal A, V)\|^2 \quad \text{where } c_s(\mathcal A, V) = \frac{1}{|A_s|}\sum_{i\in A_s} p_i.
\end{align}
Note that each $c_s(\mathcal A,V)$ is $\mathrm O(d)$-eqivariant by construction. The $\operatorname{kmeans-obj}
$ function can be  can be written as 
\begin{align} \label{eq.kmeans}
    \operatorname{kmeans-obj}(\mathcal A, V)  =\frac12\sum_{s=1}^k \frac{1}{|A_s|} \sum_{i\in A_s} \sum_{j \in A_s} \|p_i - p_j\|^2;
\end{align}
see for instance \cite{MIXON}. Note that this is an $\mathrm O(d)$-invariant polynomial function of $V$ for fixed $\mathcal A$.

We get the desired result by arguing (i) that for a fixed $\mathcal A_0\in \mathcal P_k(n)$, the centroids of the sets of columns of $V$ dictated by the parts of $\mathcal A_0$ have distinct 2-norms away from a $\mathrm O(d)\times \mfS_n$-stable measure-zero closed set $B^\dagger$, thus they can be unambiguously ordered by $2$-norm away from $B^\dagger$; (ii) for each $\mathcal A_0$ the map from $V$ to the matrix whose columns are these centroids ordered by $2$-norm 
($V\mapsto \arg\operatorname{sort}(\{\|c_s(\mathcal A_0, V)\|\}_{s=1}^d)$) 
is continuous away from this same $B^\dagger$ (note that, for fixed $V$, the $C$ of the theorem statement coincides with $\arg\operatorname{sort}(\{\|c_s(\mathcal A_0, V)\|\}_{s=1}^d)$ for $\mathcal A':=\arg\min_{\mathcal A}\operatorname{kmeans-obj}(\mathcal{A},V)$); and (iii) that there is another $\mathrm O(d)\times \mfS_n$-stable measure-zero closed set $B^\bullet$ such that for $V_0\notin B^\bullet$, the $\mathcal A'$ that minimizes $\operatorname{kmeans-obj}(\mathcal A,V)$ is constant for $V$ in a neighborhood of $V_0$, thus the continuity of the map in (ii) away from $B^\dagger$ implies the continuity of $C$ away from $B^\dagger\cup B^\bullet$. We take $B^!$ to be the union of $B^\dagger$ and $B^\bullet$. It is clear from the construction that $C$ is $\mathrm O(d)$-equivariant and $\mfS_n$-invariant where defined, so (i), (ii), and (iii) will complete the proof.

We first observe that for each $\mathcal A_0 = \{A_1,\ldots, A_k\} \in \mathcal P_k(n) $, the map  $ V \mapsto \arg\operatorname{sort}(\{\|c_s(\mathcal A_0, V)\|\}_{s=1}^d)$ (the matrix whose columns are the the vectors $\{c_s(\mathcal A_0, V)\}_{s=1}^d$ ordered by 2-norm) is continuous everywhere except for a zero-measure closed stable subset of $\mathbb R^{d\times n}$. 
This is because each of the $c_s(\mathcal A_0, V)$ ($s=1, \ldots, d$) are linear functions of $V$; thus the only discontinuity points may occur when the order by $2$-norm changes. In turn, this can only happen when $\|c_\ell (\mathcal A_0, V)\|=\|c_k(\mathcal A_0, V)\|$ for some $\ell\neq k$.
We observe that the set $\{\|c_\ell(\mathcal A_0, V)\|=\|c_k(\mathcal A_0, V)\|:V\in\mathbb R^{d\times n}\}$, or equivalently $\{\|c_\ell(\mathcal A_0, V)\|^2=\|c_k(\mathcal A_0, V)\|^2:V\in\mathbb R^{d\times n}\}$, is the vanishing set of a nonzero polynomial function $\xi_{\mathcal A_0, \ell,k}(V) := \|c_\ell(\mathcal A_0, V)\|^2-\|c_k(\mathcal A_0, V)\|^2$, and is therefore Zariski-closed. It is also $\mathrm O(d)$-stable because $c_\ell,c_k$ are $\mathrm O(d)$-equivariant. Therefore, we can take 
\[
B^\dagger = \bigcup_{\substack{\mathcal A_0\in\mathcal P_k(n) \\ \ell<k}} \xi_{\mathcal A_0, \ell, k}^{-1}(\{0\}).
\]
Note that this is a finite union. The union over $\mathcal A_0$ (which is not needed in practice but it simplifies the proof) makes this set $\mfS_n$-stable. This proves (i) and (ii).

What remains to show is that the assignment $V\mapsto \arg\min_{\mathcal A \in \mathcal P_k(n)} \operatorname{kmeans-obj}(\mathcal A, V)$ is constant everywhere except a closed $\mathrm O(d) \times \mfS_n$-stable zero-measure set. 
We first  observe that, by continuity of $\operatorname{kmeans-obj}$ in $V$, if $\operatorname{kmeans-obj}(\mathcal A_1, V_0) < \operatorname{kmeans-obj}(\mathcal A_2, V_0)$ then one can take a small open ball about $V_0$ such that for all $V$ in that open ball we have $\operatorname{kmeans-obj}(\mathcal A_1, V) < \operatorname{kmeans-obj}(\mathcal A_2, V)$.
This shows that the optimal assignment is locally constant in an open set, and the failure to be locally constant can only occur at points where 
\[
\operatorname{kmeans-obj}(\mathcal A_1, V) = \operatorname{kmeans-obj}(\mathcal A_2, V)
\]
for some $\mathcal A_1\neq \mathcal A_2$. Because $\operatorname{kmeans-obj}$ is a polynomial in $V$ for fixed $\mathcal A$, a similar argument as the one above shows that for $\mathcal A_1\neq \mathcal A_2$, 
\[
\{V\in\mathbb R^{d\times n}: \operatorname{kmeans-obj}(\mathcal A_1, V) = \operatorname{kmeans-obj}(\mathcal A_2, V)\}
\]
is a (Zariski-closed, and therefore) measure-zero set. From the $\mathrm O(d)$-invariance of \eqref{eq.kmeans}, it is also $\mathrm O(d)$-stable. Therefore, we can take 
\[
B^\bullet = \bigcup_{\{\mathcal A_1,\mathcal A_2\} \in \binom{\mathcal P_k(n)}{2}}\{V\in\mathbb R^{d\times n}: \operatorname{kmeans-obj}(\mathcal A_1, V) = \operatorname{kmeans-obj}(\mathcal A_2, V)\},
\]
where $\binom{\mathcal P_k(n)}{2}$ denotes the set of 2-subsets of $\mathcal P_k(n)$; the union makes it $\mfS_n$-stable. This proves (iii).
\end{proof}

\begin{prop}\label{prop:alternative-C}
    Let $V\in \RR^{d\times n}$ and let  $v_i$ be the $i$th column of $V\in \RR^{d\times n}$. Then one obtains $d$ equivariant identifiers (as defined above) with the map
    \[
    C:\RR^{d\times n}\rightarrow \RR^{d\times d}
    \]
    obtained by sending $V$ to the matrix $C_V\in \RR^{d\times d}$ whose $j$th column ($j=1,\dots,d$) is
    \[
    \frac{1}{n}\sum_{i=1}^{n} \|v_i\|^{2(j-1)}v_i.
    \]
\end{prop}

\begin{proof}
    The map $C$ defined in the proposition statement is a polynomial map, so it is defined and continuous everywhere. By construction it is $\mfS_n$-invariant and $\mathrm O(d)$-equivariant. We only need to check that the columns are linearly independent away from a closed, measure zero, $\mathrm O(d)\times \mfS_n$-stable bad set $\breve B\subset \RR^{d\times n}$. 
    
    We take $\breve B$ to be the vanishing set of the polynomial function $V\mapsto \det C_V$. Since $\det:\RR^{d\times d}\rightarrow \RR$ is $\mathrm O(d)$-equivariant with respect to the canonical action on $\RR^{d\times d}$ and the sign action on $\RR$, the composition $\det \circ \, C: \RR^{d\times n}\rightarrow \RR$ is $\mathrm O(d)$-equivariant and $\mfS_n$-invariant, thus $\breve B$ is $\mathrm O(d)\times \mfS_n$-stable. Because it is also Zariski closed, it is enough to show that it is proper, i.e., to exhibit a single $V$ for which $\det C_V \neq 0$. We take $V$ such that
    \[
    v_i := ie_i
    \]
    for $i=1,\dots, d$, and $0$ otherwise. Then
    \[
    C_V = \frac{1}{n}\begin{pmatrix}
        1 & 1 & \dots & 1\\
        2 & 8 & \dots & 2^{2d-1}\\
        \vdots & \vdots & \ddots & \vdots \\
        d & d^3 & \dots & d^{2d-1}
    \end{pmatrix} = \frac{1}{n}\begin{pmatrix}
        1 & & & \\
         & 2 & & \\
         & & \ddots & \\
         & & & d
    \end{pmatrix}
    \begin{pmatrix}
        1 & 1 & \dots & 1\\
        1 & 4 & \dots & 2^{2(d-1)}\\
        \vdots & \vdots & \ddots & \vdots\\
        1 & d^2 & \dots & d^{2(d-1)}
    \end{pmatrix}.
    \]
    This is a diagonal matrix times a Vandermonde matrix; nonsingularity follows from the fact that the numbers $1,\dots, d$ are all different and nonzero.
\end{proof}

\begin{rem}
    Although Theorem~\ref{thm:Ond-features} is formulated in a manner oriented to its purpose in machine learning, we draw out the consequence in the traditional setting of invariant theory, and ask a natural followup question. Because the $C$ of Proposition~\ref{prop:alternative-C} is a polynomial map, and polynomial permutation-orbit separators $D$ of size $O(n)$ do exist as well (for example, the construction in the proof of \cite[Theorem~5]{tabaghi2023universal}, with appropriate choice of $\ell$), Theorem~\ref{thm:Ond-features} gives a recipe for constructing a family of $O(n)$ {\em polynomials} that generically separates the orbits of $\mathrm{O}(d)\times \mfS_n$ on $\RR^{d\times n}$. In the setting of an algebraic group acting on a vector space (or more generally a variety) over an algebraically closed field, generic orbit separation would imply generation of the invariant field, by Rosenlicht's theorem~\cite[Theorem~2]{rosenlicht1956some}. Because $\RR$ is not algebraically closed, Theorem~\ref{thm:Ond-features} does not immediately imply a field generation result. But it is natural to ask:
\end{rem}

\begin{RQ}
    For appropriate choice of polynomial permutation-orbit separators $D$, and polynomial $d$ equivariant identifiers $C$, do the entries of $h(V)$ generate the field of rational invariants of $\mathrm{O}(d)\times \mfS_n$'s action on $\RR^{d\times n}$?
\end{RQ}

\section{Extension of DeepSets with invariant features} \label{sec:DeepSets}

The goal of equivariant ML is to learn functions on data satisfying desired symmetries. One promising approach is to design machine learning models taking invariant features as inputs (\cite{blum2022machine}). In the previous sections, we defined sets of (almost everywhere) separating invariants for point clouds and symmetric matrices. It remains to show how to incorporate these separating invariants into ML models. To this end, we will use these invariants in combination with a very popular $\mfS_n$-invariant machine learning architecture known as DeepSets \cite{zaheer2017deep}. The principal virtue of DeepSets for our purposes is that it allows to handle data (either symmetric matrices or point clouds) of varying sizes in a uniform way.

Given a subset $\mathcal X \subset \mathbb R^d$, DeepSets are parametric functions that take a finite multiset of elements in $\mathcal X$ as input and output a vector in $\mathbb R^k$ 
\[
\texttt{DeepSet}(\{x_i\in \mathcal X: i\in I\}) := \sigma\left(\sum_{i\in I} \phi(x_i)\right),
\]
where $\sigma:\mathbb R^s\to \mathbb R^k$ and $\phi: \mathbb R^d \to \mathbb R^s$ are typically (trained) standard machine learning models such as multi-layer perceptrons (denoted as $\texttt{MLP}$ below), and $I$ is an index set for the (finite) input multiset. The function $\phi$ is often called the {\em encoder}, $\sigma$ is the {\em decoder}, and $\RR^s$ is the {\em latent space}. This machine learning model was introduced in \cite{zaheer2017deep} and widely applied. Its mathematical properties, such as universality, have recently been studied \cite{wagstaff2022universal, tabaghi2023universal, amir2024neural}. 

In particular, it follows from \cite{tabaghi2023universal} that there exists a family of specific ``universal" encoders $\phi:\RR^d\rightarrow \RR^s$ with $s=2nd$ (where $n$ is an upper bound on the number of elements in the input set, and $d$ is the dimension of each of those elements) such that in the resulting DeepSets, the map $\RR^{d\times n}\rightarrow \RR^s$ to the latent space given by $V\mapsto \sum_{v\in\operatorname{cols}(V)} \phi(v)$ is a permutation-orbit separator in the sense of Section~\ref{sec:reduction}. It follows that these DeepSets are universal approximators of $\mfS_n$-invariant continuous functions in compact sets of $\mathbb R^{d\times n}$. This is interpreted as implying that DeepSets with at worst $O(n)$ latent dimension are universally expressive of $\mfS_n$-invariant functions. That said, from a practical standpoint, when learning a single function with low-dimensional output, it is overly constraining to commit at the outset to one of the ``universal" encoders $\phi$ with its latent space dimension of $2nd$. Instead, the models used in practice learn both $\sigma$ and $\phi$ from data, and may set $s\ll 2nd$, especially when the training data set is small.

We propose to use DeepSets in combination with the results of Sections~\ref{sec:invariants-on-S(n)}, \ref{sec:inv-clouds}, and \ref{sec:reduction}. For symmetric matrix data, rather than computing the $f$'s from the input $X$ by sorting, we use the sets $\{X_{ii}\}$ and $\{X_{ij}\}$ as inputs to a DeepSet model, dubbed \emph{DeepSet for Conjugation Invariance} (DS-CI):
    \begin{equation}
       \text{DS-CI}(X) = \texttt{MLP}_c\left(  \texttt{DeepSet}_1 ( \{ f^d_k(X) \}_{k=1,\ldots, n} ),  \texttt{DeepSet}_2( \{f^o_\ell (X)\}_{\ell=1,\ldots, n(n-1)/2} ), \texttt{MLP}_3 (f^\star(X))\right) . \label{eqn:model}
    \end{equation}
Since the sets we consider for this application are 1-dimensional,  other machine learning techniques are also available (for instance, turning sets of scalars into histograms by binning and then using models defined on 1-dimensional probability distributions such as histogram regressions \cite{irpino2015linear, dias2015linear}). 

In the application to point clouds, we 
implement the permutation-orbit separators $D$ as DeepSets. In particular, we propose \emph{$\mathrm O(d)$-Invariant DeepSet} (OI-DS) as follows:
\begin{equation}
       \text{OI-DS}(V) = \texttt{MLP}_o\left( \texttt{DeepSet}( \{C_V^\top \,v : \, v\in \text{cols}(V)\} ), \;  \texttt{MLP}_{4}(C_V^\top C_V)\right), \label{eqn:model.pointcloud}
    \end{equation}
where $C_V$ is a $d$-tuple of equivariant identifiers as in Section \ref{sec:reduction}. In practice, we use the matrix of $k$-means centroids (with $k=d$), ordered by norm, defined in Section \ref{sec:reduction}. Here the DeepSet model is evaluated on sets of $n$ vectors of dimension $d$ (the columns of the matrix $C_V$).

In the application to point clouds, we typically have fixed $d$ (for example, $d=3$ is common) and very large $n$. In this regime, the algebraic variety that parametrizes point clouds modulo the $\mathrm O(d) \times \mfS_n$ action is of very low dimension compared to the variety of (all) symmetric matrices $\mathcal{S}(n)$, specifically $dn - \binom{d}{2}$ vs. $\binom{n+1}{2}$. This is the underlying reason why it was possible to reduce the number of invariant features from $O(n^2)$ in \eqref{eqn:model},  to $O(n)$ (for fixed $d$) in  
\eqref{eqn:model.pointcloud}. In practice, such efficiency gain is significant as there are point cloud applications (for example in cosmology) that involve large datasets of low-dimensional point clouds.

\section{Experimental results}
\label{sec:experiments}
\subsection{Molecular property prediction on QM7b} \label{sec:QM7b}

In the first experiment, we consider predicting molecule properties given the molecular structure. The properties are invariant with respect to rotations and reflections of the molecular graph, and permuting the nucleus labels. To this end, we make use of QM7b, a standard benchmark dataset for molecule property regression \cite{blum, Montavon2013}. It consists of 7,211 molecules with 14 regression targets. For each molecule, the input feature $X$ is an $n \times n$ symmetric Coulomb Matrix (CM) \cite{rupp2012fast}---a chemical descriptor of molecule. Specifically, $X$ is a function of the 3D coordinates $\{ R_i \}_{i=1}^n, R_i \in \RR^3$ of the nuclei, and the nuclear charges $\{ Z_i \}_{i=1}^n, Z_i \in \RR$, where
\begin{equation}
    X_{ij} = \begin{cases}
    0.5 \, Z_i^{2.4}, \quad i = j \\
    \frac{Z_i Z_j}{ \| R_i - R_j \|}, \quad i\neq j. 
    \end{cases}
\end{equation}

We apply DeepSet for Conjugation Invariance (DS-CI), defined in \eqref{eqn:model}, and an extension dubbed  DS-CI+. This extension uses a ``binary expansion" \cite[Appendix B]{montavon2013machine} of the invariant features $f$, which is a common preprocessing step in molecular regression. Concretely, the binary expansion $\phi: \mathbb R \to [0,1]^{d}$ lifts a scalar to a higher-dimensional space via
 \begin{equation}
     \phi(x) = \left[\ldots, \texttt{sigmoid}\left(\frac{x-\theta}{\theta}\right), \, \texttt{sigmoid}\left(\frac{x}{\theta}\right), \, \texttt{sigmoid}\left(\frac{x+\theta}{\theta}\right),\ldots\right],
 \end{equation}
where $\texttt{sigmoid}(x) = e^x / (1+e^x)$. We choose $\theta = 1$ and $d = 100$ in DS-CI+. We use an $80/10/10$ train-validation-test split on QM7b and repeat the experiment over $10$ random data splits. We use Mean Absolute Error (MAE) as the loss function and $\texttt{Adam}$ optimizer with initial learning rate $0.01$. We train the models for at most $1000$ epochs and report the test accuracy at the model checkpoint with the best validation accuracy. Table \ref{tab:exp-QM7b} shows that our lightweight models based on the proposed invariant features achieve competitive performance.

\begin{table*}[htb]
\scriptsize
\caption{We report the mean absolute error (MAE) on the test set over 10 random data splits (80/10/10 for train/validation/test sets). The results for the Kernel Ridge Regression (KRR) and Deep Tensor Neural Network (DTNN) with the same data split ratio are taken from \cite[Table 10]{wu2018moleculenet}. Bold numbers indicate the best results.}
\label{tab:exp-QM7b}
\centering
\begin{tabular}{ccccccccc}
\toprule
\textbf{MAE} $\downarrow$	& 
\begin{tabular}{@{}c@{}}Atomization  \\ PBE0 \end{tabular} &
\begin{tabular}{@{}c@{}} Excitation \\ ZINDO \end{tabular} & 
\begin{tabular}{@{}c@{}} Absorption  \\ ZINDO \end{tabular} &
\begin{tabular}{@{}c@{}} HOMO \\ ZINDO \end{tabular} &
\begin{tabular}{@{}c@{}} LUMO \\ ZINDO \end{tabular} &
\begin{tabular}{@{}c@{}} 1st excitation \\ ZINDO \end{tabular} & \begin{tabular}{@{}c@{}} Ionization \\ ZINDO \end{tabular} 
\\	
\midrule							
 KRR \cite{wu2018moleculenet}	& 9.3	&
1.83	&
0.098	&
0.369	&
0.361	&
0.479	&
0.408	\\
 DS-CI (Ours)	& 12.849$\pm$0.757 &	1.776$\pm$0.069 &	0.086$\pm$0.003 &	0.401$\pm$0.017 &	0.338$\pm$0.048 &	0.492$\pm$0.058 &	0.422$\pm$0.012  \\
 DTNN \cite{wu2018moleculenet} 
& 21.5 &
1.26 &
0.074 &
0.192 &
0.159 &
0.296 & 
0.214 
\\
 DS-CI+ (Ours) &\textbf{ 7.650$\pm$0.399} & \textbf{1.045$\pm$0.030} & \textbf{0.069$\pm$0.005} & \textbf{0.172$\pm$0.009} &  \textbf{0.119$\pm$0.005} & \textbf{0.160$\pm$0.011} & \textbf{0.189$\pm$0.011} \\
\midrule
\textbf{MAE} $\downarrow$	&
\begin{tabular}{@{}c@{}}  Affinity \\ ZINDO \end{tabular} &
\begin{tabular}{@{}c@{}} HOMO \\ KS \end{tabular}
&
\begin{tabular}{@{}c@{}} LUMO \\ KS \end{tabular}
&
\begin{tabular}{@{}c@{}} HOMO \\ GW \end{tabular}
&
\begin{tabular}{@{}c@{}} LUMO \\ GW \end{tabular}
&
\begin{tabular}{@{}c@{}} Polarizability \\ PBE0 \end{tabular}
  &
  \begin{tabular}{@{}c@{}} Polarizability \\ SCS \end{tabular}  \\
\midrule
KRR \cite{wu2018moleculenet} &0.404	&
0.272	&
0.239	&
0.294	&
0.236	&
0.225	&
0.116	\\

 DS-CI (Ours) & 0.404$\pm$0.047 &	0.302$\pm$0.009 &	0.225$\pm$0.01 &	0.329$\pm$0.016 &	0.213$\pm$0.008 &	0.255$\pm$0.015 &	0.114$\pm$0.008 \\
 DTNN \cite{wu2018moleculenet}
& 
0.174 &
\textbf{0.155} &
\textbf{0.129} &
\textbf{0.166} &
\textbf{0.139} &
0.173 &
0.149 \\
 DS-CI+ (Ours) & \textbf{0.122$\pm$0.002} & 0.169$\pm$0.007 & 0.135$\pm$0.007 & 0.183$\pm$0.005 & \textbf{0.139$\pm$0.004} & \textbf{0.139$\pm$0.005} & \textbf{0.088$\pm$0.004} \\
 \bottomrule
 \end{tabular}
 \end{table*}

 \subsection{Point cloud distance prediction on ModelNet10}

 \begin{figure}[htb!]%
    \centering
    \subfloat[\centering Class 2 (chair)]{{\includegraphics[width=6cm]{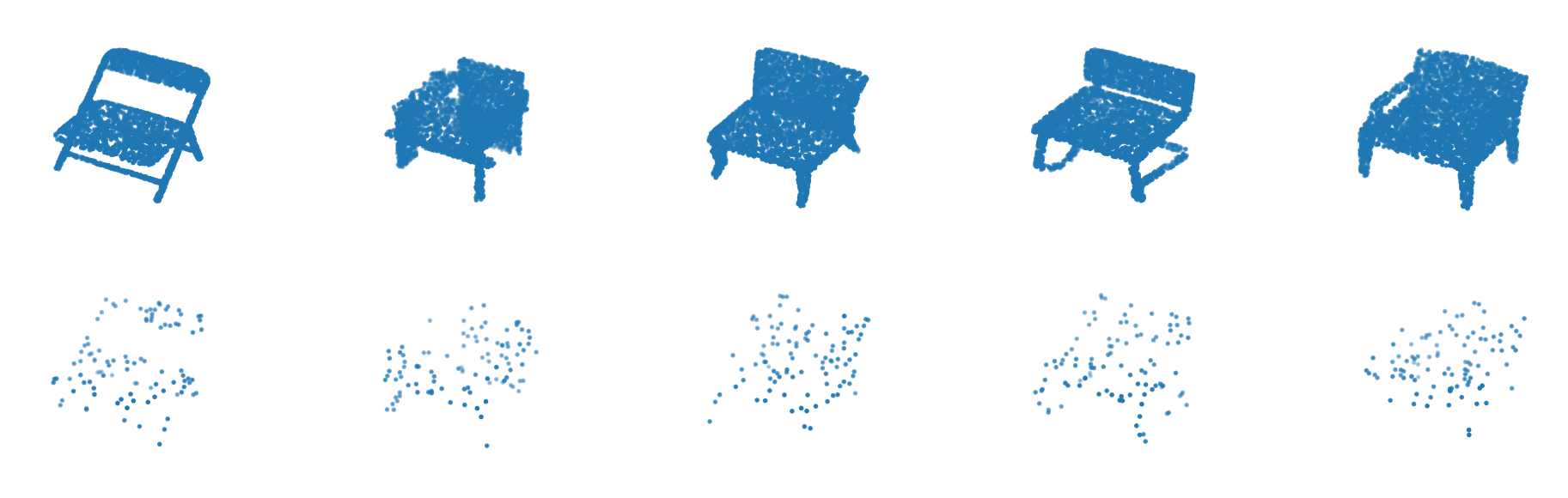} }}%
    \qquad
    \subfloat[\centering Class 7 (sofa)]{{\includegraphics[width=6cm]{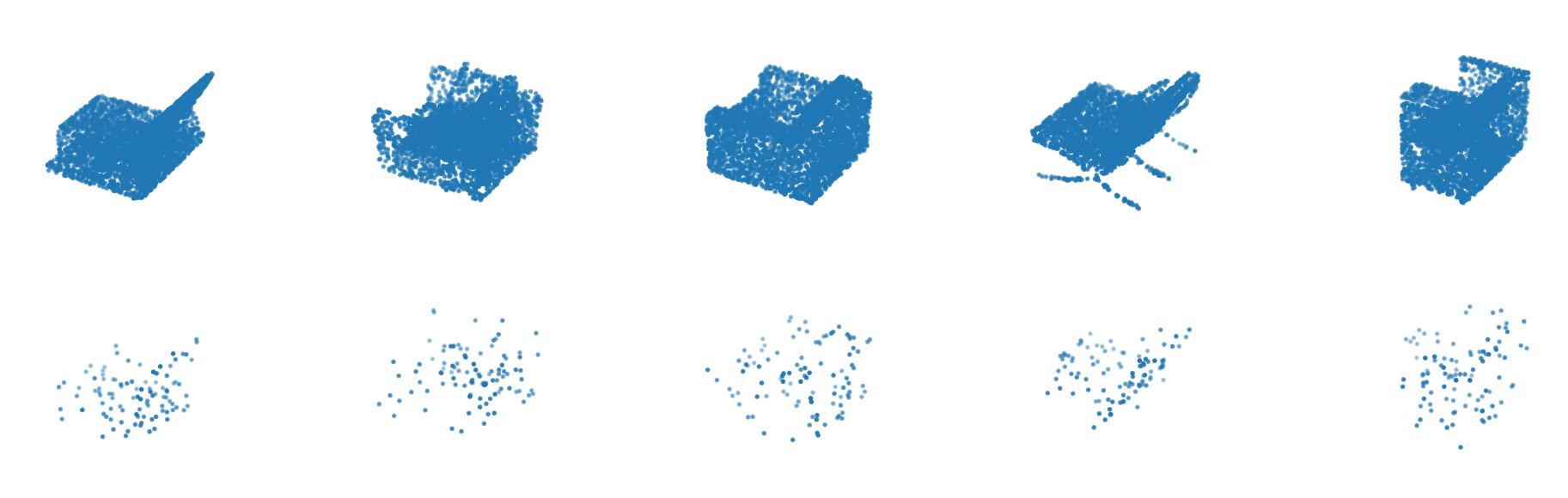} }}%
    \caption{Example point clouds from ModelNet10: (Top) original; (Bottom) downsampled with 100 random sampling points per point cloud.}%
    \label{fig:modelnet}%
\end{figure}

In the second experiment, our goal is to learn a rotation-, reflection-, and permutation-invariant function on pairs of point clouds to approximate a certain target function given in \cite{memoli2011gromov}, which is a lower bound on the Gromov-Wasserstein distance (GW). We make use of the $\mathrm O(d)$-Invariant DeepSet (OI-DS) defined in \eqref{eqn:model.pointcloud} as a feature map, followed by a distance prediction module. Specifically, given a pair of point clouds $(V, V')$, the OI-DS model outputs $\left( \text{OI-DS}_{\theta}(V), \text{OI-DS}_{\theta}(V') \right) \in (\RR^t, \RR^t)$ where $\theta$ denotes the learnable weights and $t$ is the output dimension (an engineering choice). Subsequently, we predict the Gromov-Wasserstein distance by learning
\begin{align*}
    \widehat{\text{GW}}(V, V') = a \left\| W \left(\text{OI-DS}_{\theta}(V) - \text{OI-DS}_{\theta}(V') \right) \right\|^2 + b,
\end{align*}
where $W \in \RR^{t \times t}, a \in \RR, b \in \RR$ are learnable weights. Given a set of training data $\left\{(V_i, V_j),  \text{GW}(V_i, V_j) \right\}$, we optimize the OI-DS model and the prediction module end-to-end (i.e., the learnable weights $\theta, W, a, b$) via full-batch gradient descent, and test the prediction performance on an independent test set. 

We consider some small collections of point clouds from ModelNet10, a benchmark dataset for point cloud learning tasks \cite{wu20153d}. The point clouds are generated from sampling on mesh faces of the CAD models for 3D geometric shapes (such as chairs, airplanes, etc). In this experiment, we randomly select $80$ point clouds from Class 2 (Fig~\ref{fig:modelnet}, bottom-left) and $80$ point clouds from Class 7 (Fig~\ref{fig:modelnet}, bottom-right), and randomly split each class into $40$ training samples and $40$ test samples. For each of these $160$ point clouds, we downsample it to $100$ points. For the training (respectively, test) samples, we use all $40 \times 40 = 1600$ cross-class pairs to construct the training (respectively, test) set. We use the Gromov-Wasserstein solver in \texttt{ott-jax}\footnote{Specifically, \texttt{ott.solvers.quadratic.lower\_bound} with source code \href{https://github.com/ott-jax/ott/blob/main/src/ott/solvers/quadratic/lower_bound.py}{here}.} to compute the third lower bound in \cite{memoli2011gromov} as the Gromov-Wasserstein-based distance $\text{GW}(V_i, V_j)$ for $i,j \in \{1,\dots,40\}$. We choose the lower bound distance as the target because it is invariant, deterministic, and computable in practice, whereas the actual GW distance is NP-hard to compute, and the algorithms to approximate it involve randomness depending on an initialization. The target GW-based distances of the training and test set are shown in the left subplots of Figure~\ref{fig:dist_mat} (a) and (b), respectively. 

We investigate both the DS-CI model \eqref{eqn:model} with all $f$s (middle subplots) and the OI-DS model \eqref{eqn:model.pointcloud}, using $k$-means centroids as the equivariant identifiers (right subplots) with $k=d=3$. As shown in Figure~\ref{fig:dist_mat}   and Table \ref{tab:experiments-rank}, both models predict the GW-based distances reasonably well, evidenced by a high rank correlation and low mean squared error. While the mean relative error $\frac{1}{n}\sum_{i=1}^n |\text{GW}_i - \widehat{\text{GW}}_i|/\text{GW}_i$ is relatively high, it is inflated due to many target distances $\text{GW}_i$ being close to zero. The performance degrades mildly when replacing $O(n^2)$ invariant features in DS-CI with $O(n)$ features (for fixed $d$) in OI-DS.

\begin{table}
\caption{Rank correlation, mean squared error (MSE), and relative error of the target and predicted distances.}
\label{tab:experiments-rank}
\centering
\begin{tabular}{lcccccc}
\toprule
	& \multicolumn{3}{c}{\textsc{DS-CI}}	 & \multicolumn{3}{c}{\textsc{OI-DS}}   \\ 
& 	Rank Corr. $\uparrow$ & MSE $\downarrow$ & Rel. Error $\downarrow$   & Rank Corr. $\uparrow$ & MSE $\downarrow$ & Rel. Error $\downarrow$ \\
 \midrule
 Training set & $0.88$& $0.01$ & $0.53$ &  $0.75$ & $0.04$ & $0.88$\\
 Test set & $0.86$ &$0.02$ & $0.60$ & $0.73$& $0.03$ & $0.63$ \\ 
\bottomrule
\end{tabular} 
\end{table}

 \begin{figure}[htb!]%
    \centering
    \subfloat[\centering Training set distances]{{\includegraphics[trim={0 0 0 0.75cm},clip, width=6cm]{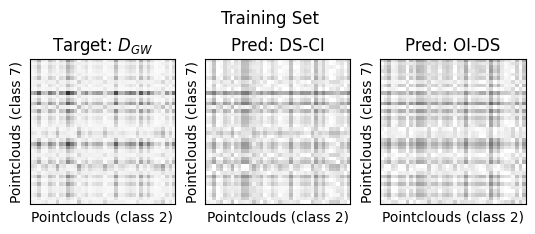} }}%
    \qquad
    \subfloat[\centering Test set distances]{{\includegraphics[trim={0 0 0 0.75cm},clip,width=6cm]{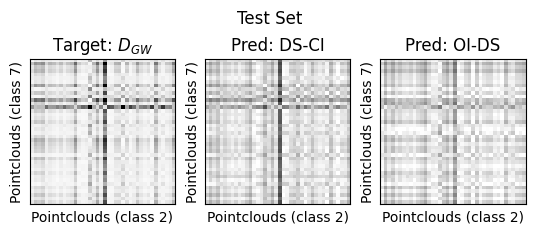} }}%
    \newline
    \subfloat[\centering Training set correlation]{{\includegraphics[trim={0 0 0 0.71cm},clip, width=4cm]{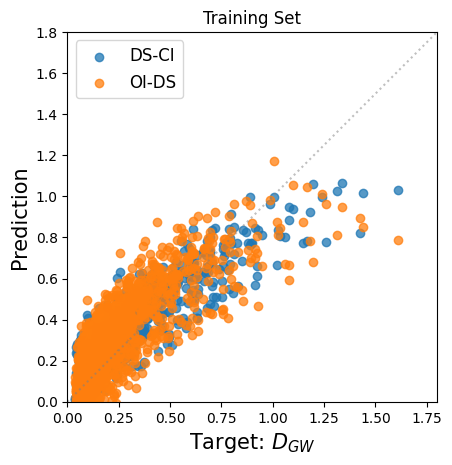} }}%
    \qquad
    \subfloat[\centering Test set correlation]{{\includegraphics[trim={0 0 0 0.71cm},clip,width=4cm]{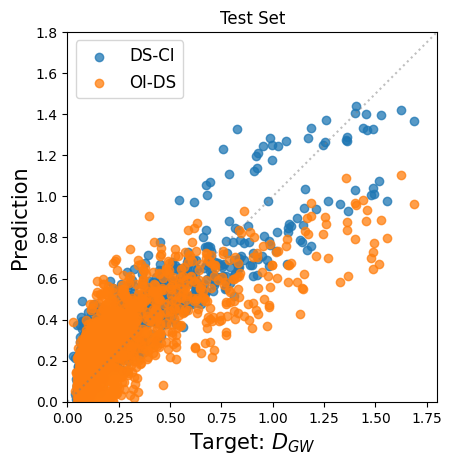} }}%
    \caption{Performance of GW-based distance regression. The cross-class pairwise distances are shown in (a) for the training set, and (b) for the test set. The left, middle, right panel in (a) and (b) corresponds to the target distances, the predicted distances from DS-CI, and the predicted distances from OI-DS, respectively.}%
    \label{fig:dist_mat}%
\end{figure}

\section{Discussion}

In this work,   we provide a general method to extend a set of generically separating invariants of a Lie group acting on a topological measure space by isometries, to a set of generically separating invariants of a subgroup of finite index. The construction is inspired by ideas from Galois theory, and it applies to classes of functions commonly used in machine learning.  

Using this construction,   we provide $O(n^2)$  functions on $n\times n$ symmetric matrices that are invariant with respect to the action of permutations by conjugation, and separate orbits outside a closed measure zero subset of the space of symmetric $n\times n$ matrices (the ``bad set''). Thus they can be used to universally approximate   continuous    invariant functions    in compact sets    outside the same bad set. We also show that when restricted to Gram matrices of point clouds in $\RR^d$, these functions again separate orbits (this time for the natural action of $\mathrm O(d)\times \mfS_n$) outside a ``bad set'' which is closed and measure zero inside the space of point clouds. Again, it follows that these functions can be used to approximate invariant   continuous    functions on point clouds.

To improve efficiency for $d$ fixed and $n$ large, we show how to use $O(n)$ invariant features to approximate any continuous invariant functions on $d$-dimensional point clouds (again   in compact sets    outside a ``bad set'' which is closed and measure zero). We demonstrate the feasibility of our approach via numerical experiments on molecular property regression and point cloud distance prediction. 

Interesting theoretical problems remain open. For instance, characterizing the nature of the ``bad set'' would allow us to understand what functions  these invariants can approximate and what functions they cannot. For this, \cite[Algorithm~3.1]{kemper2007computation} may be useful.  Another promising direction is to extend these ideas to parameterize equivariant functions on point clouds. Applications of this include learning optimal transport maps on point clouds \cite{bunne2022supervised} and equivariant diffusion models for molecule generation \cite{hoogeboom2022equivariant, xu2022geodiff}. Potentially, this could be done in a similar spirit to \cite{villar2021scalars,blum2022machine}, though those techniques would not directly apply here because the construction in \cite{blum2022machine} relies on algebra generators and this work only provides field generators.

On the practical side, it is promising to  further improve our models DS-CI and OI-DS, such as by incorporating domain-specific priors and exploring other engineering choices. One recent idea that promises stable and robust invariant machine learning models is the use of bi-Lipschitz invariants~\cite{cahill2020complete, balan2022permutation, mixon2023max, mixon2022injectivity,cahill2023bilipschitz,cahill2022group,balan2023i, balan2023g, cahill2024stable}. It would be worthwhile to investigate if a transformation that makes our lightweight invariants bi-Lipschitz functions of the data (with respect to the quotient metric) may improve the model performance in practice.

We would like to remark that we were motivated to study this problem by interactions with cosmologists, who craft invariant features to design machine learning models on point clouds (see \cite{villanueva2022learning} and \cite[Chaper 3]{kate}). Cosmological surveys like the ones coming from LSST \cite{ivezic2019lsst} are expected to be have point clouds with about $10^{10}$ points. Therefore scalability is a very important problem. A promising next step is to adapt the OI-DS model to problems in cosmology.

\section*{Acknowledgements}

We thank Alexandra Pevzner (UMN) for research  discussions at the early stages of this project, and for feedback on the proof of Proposition~\ref{thm:low-rank-separating}, including the main idea for the $d=1$ case.
We thank David W. Hogg (NYU), Kate Storey-Fisher (DIPC), Gaby Contardo (SISSA), Francisco Villascusa (Simons Foundation), and Kaze Wong (Flatiron Institute) for giving us context about point clouds in cosmology that inspired this research; some of these conversations took place at the Coworking Retreat on Equivariant Machine Learning held at Johns Hopkins University in March 2023. We also thank Puoya Tabaghi (UCSD) for conversations about \cite{tabaghi2023universal}; Facundo Mémoli (OSU) for conversations about the Gromov-Wasserstein lower bounds; Daniel Packer (OSU) for implementing the Gromov-Wasserstein lower bound for OTT-JAX that we use in our experiments; Tim Carson (Google) for conversations about $k$-means clustering; and Erik Thiede (Cornell) for suggestions on the molecule property regression experiments. SV and BBS were partially supported by ONR N00014-22-1-2126, NSF CCF 2212457, and NSF CAREER 2339682. SV is also funded by the NSF–Simons Research
Collaboration on the Mathematical and Scientific Foundations of Deep Learning
(MoDL) (NSF DMS 2031985), and NSF 2430292.

\appendix

\section{Appendix}

We give the proof of Theorem~\ref{thm:low-rank-separating}, and the original (Galois-theoretic) proof of Theorem~\ref{thm:generic-separators}. For Theorem~\ref{thm:generic-separators}, recall that the theorem states that the $f^d$'s, the $f^o$'s, and $f^\star$, defined in Section~\ref{sec:invariants-on-S(n)}, separate orbits for the action of the symmetric group $\mfS_n$ on the space $\mathcal{S}(n)$ of symmetric matrices, away from a closed, $\mfS_n$-stable, measure zero subset $B$ of ``bad" matrices. The argument makes use of the elementary symmetric polynomials $e^d_1(X),\dots,e^d_n(X)$ in $X_{11},\dots,X_{nn}$, defined by the following formal identity:
\begin{equation}\label{eq:elem-diag}
\prod_{i=1}^n \left(T-X_{ii}\right) = T^n - e^d_1(X)T^{n-1}+\dots + (-1)^ne^d_n(X).
\end{equation}
In other words, they are the coefficients (with alternating signs) of the unique monic polynomial with $X_{11},\dots,X_{nn}$ as roots. (The sign convention in the definition is immaterial to our purposes, but is extremely standard.) We also use the elementary symmetric polynomials $e^o_1(X),\dots,e^o_{n(n-1)/2}(X)$ in $X_{ij}$ for $1\leq i<j\leq n$, defined by the identity
\begin{equation}\label{eq:elem-off-diag}
\prod_{1\leq i < j\leq n}\left(T-X_{ij}\right) = T^{\binom{n}{2}} - e^o_1(X)T^{\binom{n}{2} - 1} + \dots + (-1)^{\binom{n}{2}}e^o_{\binom{n}{2}}(X).
\end{equation}
(As with the $f^d$'s and $f^o$'s, the $d$ and $o$ stand for {\em diagonal} and {\em off-diagonal}. We will refer to the $e^d$'s and $e^o$'s collectively as {\em the $e$'s.})

The argument has three components:
\begin{itemize}
    \item The $f^d$'s contain the same information as the $e^d$'s, thus they have the same separation properties; similarly for the $f^o$'s and the $e^o$'s. This is in essence something that was known by the 18th century, but is noted in a context related to the present one in \cite[Example~2]{olver2023invariants}.
    \item A Galois-theoretic argument, following the main idea of \cite[Theorem~11.2]{thiery2000algebraic}, shows that the $e^d$'s and $e^o$'s (i.e., the $e$'s), together with $f^\star$, generate the field of invariant rational functions.
    \item A (version of a) classical theorem of Rosenlicht \cite[Theorem~2]{rosenlicht1956some} promises the existence of the group-invariant, proper algebraic subvariety $B$ such that generators for the field of invariant rational  functions separate orbits on its complement. 
\end{itemize}
Here are the details.

\begin{proof}[Proof of Theorem~\ref{thm:generic-separators}]
Because $f^d_1(X),\dots,f^d_n(X)$ form a sorted list of the numbers $X_{11},\dots,X_{nn}$, the values of the $f^d$'s determine the values of the $e^d$'s by substituting them into the left side of \eqref{eq:elem-diag} and expanding the product. Similarly, the $e^d$'s determine the values of the $f^d$'s by substituting into the right side of \eqref{eq:elem-diag}, finding the roots of the polynomial, and sorting them. In a similar way, the $f^o$'s and $e^o$'s determine each other by substitution into \eqref{eq:elem-off-diag}. Thus, there is an $f^d$, respectively an $f^o$, that distinguishes two orbits of $\mfS_n$ on $\mathcal{S}(n)$ if and only if there is an $e^d$, respectively an $e^o$, that distinguishes them. 

It follows that the $f^d$'s, the $f^o$'s, and $f^\star$ (taken as a whole) distinguish a pair of orbits if and only if the $e^d$'s, the $e^o$'s, and $f^\star$ (taken as a whole) distinguish the same pair of orbits. Thus, it suffices to prove the theorem with $e^d$'s and $e^o$'s in the place of $f^d$'s and $f^o$'s respectively, and this is what we will do.

\begin{lemma}\label{lem:field-generation}
The $e^d$'s, the $e^o$'s, and $f^\star$, generate the field of rational functions on $\mathcal{S}(n)$ that are invariant under the action of $\mfS_n$ considered here.
\end{lemma}

The proof of the lemma follows the idea of \cite[Theorem~11.2]{thiery2000algebraic}.

\begin{proof}
    An element of $\mathcal{S}(n)$ is specified by its diagonal and above-diagonal entries because it is symmetric. We take the diagonal entries $X_{ii}$ and above-diagonal entries $X_{ij}$ ($i<j$) as coordinate functions on the real vector space $\mathcal{S}(n)$. The below-diagonal entries $X_{ji}$ ($i<j$) are identified with the above-diagonal entries via $X_{ji}=X_{ij}$. In what follows, we will use $X_{ij}$ and $X_{ji}$ synonymously.

    The action of $\mfS_n$ on $\mathcal{S}(n)$ with which we are concerned in this work applies a permutation $\pi$ to the underlying index set $\{1,\dots,n\}$ over which $i$ and $j$ range, and then acts on the above functions by sending $X_{ij}$ to $X_{\pi(i)\pi(j)}$ (bearing in mind the above identification of each $X_{ij}$ with the corresponding $X_{ji}$).

    Meanwhile, let the group $\Gamma := \mfS_n\times \mfS_{n(n-1)/2}$ act on these functions as follows. Fix some arbitrary order on the above-diagonal entries. Then, given $(\sigma,\tau)\in \Gamma$, permute the diagonal entries $X_{ii}$ according to $\sigma$ and the above-diagonal entries $X_{ij}$ ($i<j$) according to $\tau$ (based on the fixed order). Of course, the below-diagonal entries are permuted correspondingly, since they are identified with the above-diagonal entries.

    The action of $\mfS_n$ on $\mathcal{S}(n)$ with which we are concerned then arises from embedding $\mfS_n$ as that subgroup $G$ of $\Gamma$ consisting of pairs $(\sigma,\tau)$ such that the actions of $\sigma$ and $\tau$ on the $X_{ii}$'s and $X_{ij}$'s, respectively, are induced by the same underlying permutation $\pi$ of the indices $i,j$.

    We claim that a pair $(\sigma,\tau)\in \Gamma$ is induced by the same underlying permutation $\pi$ of the indices---i.e., it lies in $G$---if and only if it preserves the relation of sharing an index between the diagonal entries $X_{ii}$ and the off-diagonal entries $X_{ij}$ ($i\neq j$). Viewing the diagonal entries $X_{ii}$ as corresponding with the nodes of a graph, while the off-diagonal entries correspond with the edges, the claim is that a pair of permutations $\sigma$ of the nodes and $\tau$ of the edges is induced by a single node permutation $\pi$, if and only if the pair $(\sigma,\tau)$ preserves the incidence relation between nodes and edges.
    
    We see this as follows. In one direction, suppose that $\sigma$ and $\tau$ are both induced by $\pi$. Then $\sigma(X_{ii}) = X_{\pi(i)\pi(i)}$, while $\tau(X_{ij}) = X_{\pi(i)\pi(j)}$, so $\sigma(X_{ii})$ and $\tau(X_{ij})$ share the index $\pi(i)$. Conversely, suppose that the pair $(\sigma,\tau)$ preserves the relation of sharing an index. We first define a permutation $\pi$ by the relation $\sigma(X_{ii}) = X_{\pi(i)\pi(i)}$ (note this is a legitimate definition because $\sigma$ sends diagonal entries to diagonal entries). By construction, $\pi$ induces $\sigma$. Then we verify that $\pi$ induces $\tau$, as follows. We know by assumption that $\tau(X_{ij})$ shares indices with both $\sigma(X_{ii})=X_{\pi(i)\pi(i)}$ and $\sigma(X_{jj})=X_{\pi(j)\pi(j)}$. These indices are distinct since $\sigma$ is a permutation and $i\neq j$ by assumption. But then the 2-set of them must be $\{\pi(i), \pi(j)\}$, by construction of $\pi$. So we conclude that $\tau(X_{ij}) = X_{\pi(i)\pi(j)} $ (invoking, if necessary, the identification of $X_{\pi(i)\pi(j)}$ with $X_{\pi(j)\pi(i)}$). This proves the claim.

    Since $f^\star(X)$ is the sum of all and only those products of a diagonal entry $X_{ii}$ and an off-diagonal entry $X_{ij}$ that share an index, it follows immediately from the just-established claim that an element $(\sigma,\tau)\in \Gamma$ leaves $f^\star(X)$ invariant if and only if it lies in $G$.

    We now involve Galois theory. Let $K := \RR(\{X_{ii}\},\{X_{ij}\})$ be the field of all rational functions on $\mathcal{S}(n)$. By the Fundamental Theorem on Symmetric Rational Functions applied twice, the fixed field of $\Gamma$'s action on $K$ is exactly the subfield generated by the $e^d$'s and $e^o$'s:
    \[
    K^\Gamma = \RR(e^d_1(X),\dots,e^d_n(X),e^o_1(X),\dots,e^o_{n(n-1)/2}(X)).
    \]
    The field extension $K/K^\Gamma$ is Galois, with Galois group $\Gamma$. By the previous paragraph, the elements of $\Gamma$ that pointwise-fix the subfield $K^\Gamma(f^\star(X))$ generated over $K^\Gamma$ by $f^\star(X)$ are exactly those that lie in $G$. It follows from the Fundamental Theorem of Galois Theory that $K^\Gamma(f^\star(X))$ is precisely the fixed field of $G$, and we have
    \begin{align*}
        K^G &= K^\Gamma(f^\star(X))\\
        &= \RR(e^d_1(X),\dots,e^d_n(X),e^o_1(X),\dots,e^o_{n(n-1)/2}(X), f^\star(X)).
    \end{align*}
    This proves the lemma.
\end{proof}

The theorem now follows immediately from \cite[Corollary~3.31]{bandeira2017estimation} (applied with $U$ as the linear span of the $e^d$'s, the $e^o$'s, and $f^\star$). The conclusion is that the map
\begin{align*}
    \mathcal{S}(n) &\rightarrow \RR^{n+\binom{n}{2}+1} \\
    X &\mapsto \left(e^d_1(X),\dots,e^d_n(X),e^o_1(X),\dots,e^o_{n(n-1)/2}(X),f^\star(X)\right)
\end{align*}
defines an $\mfS_n$-invariant polynomial map that separates all of the orbits in the complement of an $\mfS_n$-stable set $B$ of $\mathcal{S}(n)$ cut out by a nontrivial polynomial vanishing condition. (This is essentially {\em Rosenlicht's Theorem} \cite[Theorem~2]{rosenlicht1956some}, see also \cite[Theorem~2.3]{popov1994invariant}.) The set $B$ is closed and measure zero in $\mathcal{S}(n)$, with respect to Lebesgue measure, because vanishing sets of nonzero polynomials have these properties.
\end{proof}

\begin{rem}
    There is, perhaps, a family resemblance between the method of proof of the lemma used here (which follows \cite[Theorem~11.2]{thiery2000algebraic}, as mentioned above), and that of \cite[Theorem~1.2(2)]{balan2022permutation}. The latter, like our use of $f^*$, involves augmenting lists of numbers with additional data in order to force otherwise-independent permutations of the lists to be coordinated with each other. We thank Dustin Mixon for calling our attention to this. The proof of \cite[Theorem~5]{tabaghi2023universal} is another example.
\end{rem}

\begin{rem}\label{rem:rosenlicht}
    Although we are using the black box of \cite[Corollary~3.31]{bandeira2017estimation} in the above, we outline a proof for the interested reader, which will make clear where the set $B$ of ``bad" matrices comes from. The lemma shows that the $e$'s and $f^\star$ generate the field of invariant rational functions on $\mathcal{S}(n)$. This field contains the {\em ring} of invariant {\em polynomial} functions $\RR[\{X_{ii}\},\{X_{ij}\}]^G$. By general theory, this ring is finitely generated as an algebra (e.g., \cite[Theorem~2.2.10]{derksen2015computational} or \cite[Theorem~A.31]{bandeira2017estimation}) and separates all orbits (e.g., \cite[Chapter~3, \S~4, Theorem~3]{onishchik2012lie} or \cite[Theorem~A.32]{bandeira2017estimation}); in particular, there is a finite set of invariant polynomials $F_1,\dots,F_m$ on $\mathcal{S}(n)$  that separates all orbits, namely the algebra generators. Because the $e$'s and $f^\star$ generate the field of invariant rational functions, which contains the $F_j$'s, there exist rational functions $P_1/Q_1,\dots,P_m/Q_m$ in $n + n(n-1)/2 + 1$ indeterminates, such that
\[
F_j(X) = \frac{P_j(e^d_1(X),\dots,e^d_n(X),e^o_1(X),\dots,e^o_{n(n-1)/2}(X), f^\star(X))}{Q_j(e^d_1(X),\dots,e^d_n(X),e^o_1(X),\dots,e^o_{n(n-1)/2}(X), f^\star(X))}
\]
as rational functions, for each $j=1,\dots, m$. Since the $F_j$'s separate all orbits, the $e$'s and $f^\star$ must separate orbits as well, as long as none of the denominators in these expressions vanish. Thus, $B$ can be taken to be the locus of those $X$ in $\mathcal{S}(n)$ for which any $Q_j(e^d_1(X),\dots,e^d_n(X),e^o_1(X),\dots,e^o_{n(n-1)/2}(X), f^\star(X))$ vanishes.
\end{rem}

We turn to Theorem~\ref{thm:low-rank-separating}. The idea is essentially the same as for Theorem~\ref{thm:generic-separators}; however, we need a little bit of additional work to handle the fact that the Gram matrices of point clouds of dimension $d<n$ belong to a small proper subvariety of $\mathcal{S}(n)$, in general much smaller than the bad set $B$ of Theorem~\ref{thm:generic-separators}.

\begin{proof}[Proof of Theorem~\ref{thm:low-rank-separating}]

The argument has the same structure as the proof of Proposition~\ref{thm:generic-separators}. We replace the $\tilde f^d$'s and $\tilde f^o$'s with $\tilde e^d$'s and $\tilde e^o$'s, where, as in Section~\ref{sec:inv-clouds}, the tilde indicates composition with the ``Gram matrix" map 
\begin{align*}
\widehat \gamma: \RR^{d\times n} & \rightarrow \mathcal S(n)\\
V&\mapsto V^\top V.
\end{align*}
For example,
\[
\tilde f^\star = \sum_{i\neq j}(v_i^\top v_i)(v_i^\top v_j),
\]
and
\[
\tilde e^d_3 = \sum_{i<j<k}(v_i^\top v_i)(v_j^\top v_j)(v_k^\top v_k).
\]
The $\tilde e$'s contain the same information, and thus have the same power to separate orbits, as the $\tilde f^d$'s and $\tilde f^o$'s. Thus, if we can show Theorem~\ref{thm:low-rank-separating} holds with the $\tilde e$'s in the place of the $\tilde f^o$'s and $\tilde f^d$'s, then it also holds as written.

To obtain a version of Theorem~\ref{thm:low-rank-separating} with the $\tilde e$'s in place of the $\tilde f^o$'s and $\tilde f^d$'s, our goal will be to show that $\tilde f^\star$ and the $\tilde e$'s generate the field of rational functions on $\RR^{d\times n}$ that are invariant under the $\mathrm O(d)\times \mfS_n$ action. The desired statement---i.e., the existence of a closed, measure-zero set $\tilde B\subset \RR^{d\times n}$ on the complement of which any two $\mathrm O(d)\times \mfS_n$-orbits are separated by some $\tilde f$---then follows by \cite[Corollary~3.31]{bandeira2017estimation}  (or by Rosenlicht's theorem \cite[Theorem~2]{rosenlicht1956some}, or by the proof sketch in Remark~\ref{rem:rosenlicht}, which goes through unchanged in this setting). 

We turn to the task of establishing that $\tilde f^\star$ and the $\tilde e$'s  generate the field of $\mathrm O(d)\times \mfS_n$-invariant rational functions on $\RR^{d\times n}$.  We work with the induced action of $\mfS_n$ on the field of rational $\mathrm O(d)$-invariants.

Let $R:=\RR[\{X_{ii}\},\{X_{ij}\}]$ be the ring of polynomial functions on $\mathcal{S}(n)$. Let $I_{d+1}$ be the ideal in $R$ generated by the $(d+1)\times (d+1)$ minors, and consider the quotient ring $R/I_{d+1}$. We use overlines to indicate the images of elements of $R$ in $R/I_{d+1}$, for example $\overline X_{ii}, \overline f^\star, \overline e^o_{ij}\in R/I_{d+1}$.

There is a natural map
\begin{align*}
\widehat \gamma^*:R &\rightarrow \RR[\RR^{d\times n}]\\
\overline X_{ij} &\mapsto v_i^\top v_j
\end{align*}
that dualizes $\widehat \gamma$. Here, $v_i$ and $v_j$ are the $i$th and $j$th columns of the matrix $V\in \RR^{d\times n}$; $i$ may equal $j$ or not. Reflecting the permutation-equivariance of $\widehat \gamma$, this map is evidently equivariant with respect to the action of $\mfS_n$ on $R$ by simultaneous row and column permutation, and on $\RR[\RR^{d\times n}]$ by column permutation. Furthermore, it lands inside the invariant ring $\RR[\RR^{d\times n}]^{\mathrm O(d)}$, reflecting the $\mathrm O(d)$-invariance of $\widehat \gamma$. By the First \cite[Theorem~2.9.A]{weyl1946classical} and Second \cite[Theorem~2.17.A]{weyl1946classical} Fundamental Theorems for the invariant theory of $\mathrm O(d)$, the ring $\RR[\RR^{d\times n}]^{\mathrm O(d)}$ of $\mathrm O(d)$-invariant polynomial functions on $\RR^{d\times n}$ is isomorphic to $R/I_{d+1}$, and in fact, the isomorphism is induced by $\widehat \gamma^*$. In other words, $\ker \widehat \gamma^* = I_{d+1}$ and $\widehat \gamma^*$ is surjective onto $\RR[\RR^{d\times n}]^{\mathrm O(d)}$. Abusing notation, we denote the induced $\mfS_n$-equivariant isomorphism
\[
R/I_{d+1}\xrightarrow{\cong} \RR[\RR^{d\times n}]^{\mathrm O(d)}
\]
also by $\widehat \gamma^*$. Taking $\mfS_n$-invariants, we get an isomorphism
\begin{equation}\label{eq:gram-iso}
(R/I_{d+1})^{\mfS_n} \xrightarrow{\cong} \RR[\RR^{d\times n}]^{\mathrm O(d)\times \mfS_n},
\end{equation}
which we again call $\widehat \gamma^*$. Note that we have
\[
\widehat \gamma^*(\overline f^\star) = \tilde f^\star, \ \widehat \gamma^*(\overline e_k^d) = \tilde e_k^d,\text{ and }\widehat \gamma^*(\overline e_\ell^o)=\tilde e_\ell ^o
\]
for $k=1,\dots,n$ and $\ell = 1,\dots,\binom{n}{2}$.

For a linear action of a compact Lie group on a real vector space, the fraction field of the invariant ring is equal to the field of invariant rational functions (see for example the proof of \cite[Proposition~3.19]{bandeira2017estimation}). Thus, invoking the isomorphism \eqref{eq:gram-iso}, we may prove the desired result about $\tilde f^\star$ and the $\tilde e$'s generating the field of $\mathrm O(d)\times \mfS_n$-invariant rational functions on $\RR^{d\times n}$ by instead working with $\overline f^\star$ and the $\overline e$'s, and showing they generate the fraction field of $(R/I_{d+1})^{\mfS_n}$. Let $L :=\operatorname{Frac}(R/I_{d+1})$ be the fraction field of $R/I_{d+1}$. For finite groups, taking invariants commutes with taking fraction fields, so the field we hope to generate with these invariants is $L^{\mfS_n}$.

The action of $\mfS_n$ on $L$ in the previous paragraph is by simultaneous row and column permutations of the matrix $\overline X := (\overline X_{ij})$. In the proof of Theorem~\ref{thm:generic-separators} (in particular, in Lemma~\ref{lem:field-generation}), we used the symbol $G$ to refer to $\mfS_n$ in the analogous action on $R$, to distinguish it as a particular subgroup of the permutation group $\Gamma\cong \mfS_n\times \mfS_{n(n-1)/2}$ which acts independently on the $X_{ii}$'s and $X_{ij}$'s. We retain that notation here; in this language, what we must do is to show that $\overline f^\star$ and the $\overline e$'s generate $L^G$ as a field over $\RR$.

Let $K$ be the subfield of $L$ generated over $\RR$ by the $\overline e$'s. Since the $\overline e$'s are $G$-invariant, we have $K\subset L^G$. The univariate polynomials
\[
P:=T^n - \overline e^d_1T^{n-1}+\dots+ (-1)^n \overline e^d_n
\]
and
\[
Q:=T^{n(n-1)/2}-\overline e^o_1T^{n(n-1)/2-1}+\dots+(-1)^{n(n-1)/2}\overline e^o_{n(n-1)/2}
\]
have coefficients in $K$, and distinct roots $\{\overline X_{ii}\}$ and $\{\overline X_{ij}\}$ in $R/I_{d+1}$ which generate the latter as a ring over $\RR$, and thus $L$ as a field over $K\supset \RR$. Therefore, $L$ is a splitting field for the separable polynomial $PQ$ over $K$; in particular, $L/K$ is a Galois extension. Because $K\subset L^G$, we have $G\subset \operatorname{Gal}(L/K)$. Since $\overline f^\star$ is $G$-invariant, $K(\overline f^\star)\subset L^G$. Our goal is to show that $K(\overline f^\star)$ is equal to $L^G$, i.e., that $K(\overline f^\star)$ is the subfield corresponding to $G\subset \operatorname{Gal}(L/K)$ in the Galois correspondence.

Now the permutation group $\Gamma$ defined in the proof Lemma~\ref{lem:field-generation} acts on the ring $R$, but it does not in general stabilize the ideal $I_{d+1}$, so it does not act on $R/I_{d+1}$ or $L$. However, any Galois automorphism of $L/K$ fixes the coefficients of the polynomials $P$ and $Q$, and thus it separately permutes their roots. Since these roots generate $L$ over $K$, it follows that the Galois group $\operatorname{Gal}(L/K)$ is in a natural way a {\em subgroup} of $\Gamma$ (that contains $G$).\footnote{It is actually plausible that, for $d$ not too close to $0$ or $n$, $\operatorname{Gal}(L/K)$ is already equal to $G$, i.e., the only elements of $\Gamma$ that fix $I_{d+1}$ are those belonging to $G$ already. A closely related statement is proven in \cite[Lemma~2.4]{boutin2004reconstructing}. If this holds, it would mean that the field $K$ generated by the $\overline e$'s is already equal to $L^G$ without requiring the adjunction of $\overline f^\star$.}

Since we already know $K(\overline f^\star)\subset L^G$, it will follow from the Fundamental Theorem of Galois Theory that $K(\overline f^\star)=L^G$  if we show that $\overline f^\star$ is fixed by {\em only} those elements of $\operatorname{Gal}(L/K)$ lying in $G$. Since we have $\operatorname{Gal}(L/K)\subset \Gamma$, this in turn will follow if we show that for $\gamma \in \Gamma$, \[
\gamma f^\star - f^\star \in I_{d+1}
\]
implies $\gamma\in G$. We already know from the proof of Lemma~\ref{lem:field-generation} that $\gamma f^\star - f^\star = 0$ implies $\gamma \in G$; we proceed by showing that $\gamma f^\star - f^\star \in I_{d+1}$ implies that $\gamma f^\star - f^\star = 0$.

If $d>1$, this follows from the fact that $I_{d+1}$ is a homogeneous ideal of $R$ generated by forms of degree $d+1>2$, so the intersection of $I_{d+1}$ with the degree 2 component $R_2$ of $R$ is trivial. Meanwhile, $\gamma f^\star - f^\star$ is homogeneous of degree $2$. Thus, in this case,
\[
\gamma f^\star - f^\star \in I_{d+1} \Rightarrow \gamma f^\star - f^\star \in I_{d+1}\cap R_2 = \{0\}.
\]

If $d=1$, we need only slightly more work. In this case, the ideal $I_{d+1}=I_2$ is generated by binomials of the form
\[
X_{ij}X_{k\ell} - X_{i\ell}X_{kj},
\]
where some of $i,j,k,\ell$ may be equal; in fact, a subset of these binomials forms a Gr\"{o}bner basis for $I_2$ \cite{conca1994symmetric}. Since the multisets of indices appearing in the monomials $X_{ij}X_{k\ell}$ and $X_{i\ell}X_{kj}$ are equal, it follows that reduction to normal form with respect to this Gr\"{o}bner basis cannot change the multiset of indices appearing in any quadratic monomial. In particular, if $\gamma f^\star - f^\star \in I_2$, then every monomial appearing in $\gamma f^\star$ has the same multiset of indices as some monomial appearing in $f^\star$. But the monomials of $f^\star = \sum_{i\neq j} X_{ii}X_{ij}$ exhaust the set of quadratic monomials whose multiset of indices has the form $\{i,i,i,j\}$ with $i\neq j$. So $\gamma f^\star - f^\star \in I_2$ implies that each term of $\gamma f^\star$ is a term of $f^\star$, in which case  
the set of terms in $\gamma f^\star$ and in $f^\star$ coincide as sets. So actually, $\gamma f^\star - f^\star \in I_2$ implies $\gamma f^\star - f^\star = 0$ after all. This completes the $d=1$ case, and thus the proof.
\end{proof}

 We isolate the invariant-theoretic content of this proof:

\begin{corollary}\label{cor:pt-cloud-field-generators}
    The $\tilde e$'s and $\tilde f^\star$  generate the field
    \[
    \RR\left(\RR^{d\times n}\right)^{\mathrm{O}(d)\times \mfS_n}
    \]
    of rational functions on $\RR^{d\times n}$ invariant under the left action by $\mathrm{O}(d)$ and the right action by $\mfS_n$.\qed
\end{corollary}

We thank Alexandra Pevzner for the main idea for handling the $d=1$ case in  this proof  .

\bibliographystyle{alpha}
\bibliography{biblio}

\end{document}